\documentclass[journal]{IEEEtran}
\usepackage{epstopdf}
\usepackage{lipsum} 
\usepackage{multicol}
\usepackage{microtype}
\usepackage{graphicx}
\usepackage{subfigure}
\usepackage{booktabs} 
\usepackage[square,sort,comma,numbers]{natbib}
\usepackage{amsmath}
\usepackage{mathrsfs}
\usepackage{amssymb}
\usepackage{color}
\usepackage{xcolor}
\usepackage{multirow}
\usepackage{paralist}
\usepackage{verbatim}
\usepackage{algorithm}
\usepackage{algorithmic}
\usepackage{galois}
\usepackage{boxedminipage}
\usepackage{accents}
\usepackage{stmaryrd}
\usepackage[bottom]{footmisc}
\definecolor{light-gray}{gray}{0.85}
\usepackage{colortbl}
\usepackage{makecell}
\usepackage{hhline}
\usepackage{bbm}
\usepackage{mathtools}
\usepackage{xcolor}         
\usepackage{bbold}
\usepackage[utf8]{inputenc} 
\usepackage[T1]{fontenc}    
\usepackage{hyperref}       
\usepackage{url}            
\usepackage{booktabs}       
\usepackage{amsfonts}       
\usepackage{nicefrac}       
\usepackage{microtype}      
\usepackage{MnSymbol}
\usepackage{xr}

\newcommand{\norm}[1]{\left\lVert#1\right\rVert}
\newcommand\inner[2]{\langle #1, #2 \rangle}

\newcommand{\EE}{\mathbb{E}}

\newcommand{\argmin}{\mathop{\rm argmin}}

\newcommand{\ba}{\begin{array}}
\newcommand{\ea}{\end{array}}

\newtheorem{theorem}{Theorem}
\newtheorem{lemma}{Lemma}
\newtheorem{assumption}{Assumption}

\newtheorem{remark}{Remark}

\usepackage{amsopn}

\newcommand{\myred}[1]{\textcolor{black}{#1}}
\newcommand{\mynewred}[1]{\textcolor{black}{#1}}
\newcommand{\hyuninadd}[1]{\textcolor{black}{#1}}
\newcommand{\first}[1]{\textcolor{black}{#1}}
\newcommand{\hl}[1]{\textcolor{blue}{#1}}

\begin{document}
%
\title{Beyond Exact Gradients: Convergence of Stochastic Soft-Max Policy Gradient Methods with Entropy Regularization}
%
%
%

\author{Yuhao Ding,  Junzi Zhang, Hyunin Lee, and Javad Lavaei 
\thanks{Y. Ding and J. Lavaei are with the Department of Industrial Engineering and Operations Research, University of California, Berkeley, CA 94709 USA (e-mail: yuhao\_ding@berkeley.edu; lavaei@berkeley.edu).}
\thanks{J. Zhang is with Citadel Securities (work done prior to joining Citadel Securities), Chicago, IL 60603 USA (e-mail: saslascroyale@gmail.com).}
\thanks{H. Lee is with the Department of Mechanical Engineering, University of California, Berkeley, CA 94709 USA (e-mail: hyunin@berkeley.edu). }
}

%
%

%

\maketitle

\begin{abstract}
 Entropy regularization is an efficient technique for encouraging exploration and preventing a premature convergence of (vanilla) policy gradient methods in reinforcement learning (RL). However, the theoretical understanding of entropy regularized RL algorithms has been limited.  In this paper, we revisit the classical entropy regularized  policy gradient methods with the soft-max policy parametrization, whose convergence has so far only been established assuming access to exact gradient oracles. To go beyond this scenario, we  propose the first set of (nearly) unbiased stochastic policy gradient estimators with trajectory-level entropy regularization, with one being an unbiased visitation measure-based estimator and the other one being a nearly unbiased yet more practical trajectory-based estimator. We prove  that although the estimators themselves are unbounded in general due to the additional logarithmic  policy rewards introduced by the entropy term, the variances are uniformly bounded. We then propose a two-phase stochastic policy gradient (PG) algorithm that uses a large batch size in the first phase to overcome the challenge of the stochastic approximation due to the non-coercive landscape, and uses a small batch size in the second phase by leveraging the curvature information around the optimal policy.  We establish a global optimality convergence result and a sample complexity of $\widetilde{\mathcal{O}}(\frac{1}{\epsilon^2})$ for the proposed algorithm.  Our result is the first global convergence and sample complexity results for the stochastic entropy-regularized vanilla PG method.
\end{abstract}

\begin{IEEEkeywords}
Reinforcement learning,  policy gradient, stochastic approximation 
\end{IEEEkeywords}

%
\IEEEpeerreviewmaketitle
%
%
%
%

\section{Introduction}\label{sec:Intro}
Entropy regularization is a popular technique to encourage exploration and prevent premature convergence for reinforcement learning (RL) algorithms. It was originally proposed in \cite{williams1991function} to improve the performance of REINFORCE, a classical family of vanilla policy gradient (PG) methods widely used in practice. Since then, the entropy regularization technique has been applied to a large set of other RL algorithms, including actor-critic \cite{mnih2016asynchronous,haarnoja2018soft},   Q-learning \cite{o2016combining, haarnoja2017reinforcement} and trust-region policy optimization methods \cite{zang2020teac}. 
It has been shown that the entropy regularization works satisfactorily with deep learning approximations for achieving an impressive empirical performance boost, provides a substantial improvement in exploration and robustness \cite{ziebart2010modeling,haarnoja2017reinforcement,haarnoja2018soft}, and connects the  policy gradient with Q-learning under a one-step entropy regularization \cite{o2016combining} or a trajectory-level KL regularization\footnote{Note that this is related to but different from the widely-used trajectory-level entropy regularization later introduced in \cite{haarnoja2017reinforcement}.}  \cite{schulman2017equivalence}.


There has been considerable interest in the theoretical understanding of how the entropy regularization exploits the geometry of the optimization landscape. In particular, it has been shown in \cite{mei2020global, lan2021policy, cen2020fast} that entropy regularization makes the regularized objective behave similar to a local quadratic function and thus accelerates the convergence of entropy-regularized PG algorithms. \mynewred{When the exact entropy-regularized PG is available, a linear convergence rate has been established for the entropy-regularized PG algorithms with the natural PG (NPG) or policy mirror descent \cite{cen2020fast, lan2021policy} or without the NPG \cite{mei2020global}.
However, in practice, the agent does not have access to the exact entropy-regularized PG but only its stochastic estimation from the samples of trajectories. 
The advantages of entropy regularization
have mostly been established for the \mynewred{exact} gradient setting. It is not fully understood whether these advantages are only restricted to theoretical analysis in the exact gradient settings and whether any geometric property can be exploited to accelerate convergence to global optimality in the stochastic gradient settings. Recently,
it is proven in  \cite{cen2020fast} that the NPG with the entropy regularization has a sample complexity of $\widetilde{\mathcal{O}}(\frac{1}{\epsilon^2})$ in the stochastic gradient settings, where the inexactness of the gradient can be reduced to the inexactness of the  state-action value functions.}
However, the literature on the global optimality convergence and the sample complexity of the most fundamental PG, namely REINFORCE and its variants with regularizations, is still limited, despite its simplicity and popularity in practice.
The work \cite{mei2020global} has recently developed the first set of global convergence results for PG, which focuses on the soft-max policy parametrization by assuming access to exact PG evaluations. 
However, their result heavily relies on the access to the exact PG evaluations, and it has been shown that the geometric advantages existing in the exact gradient setting may not be preserved in the stochastic setting \cite{chung2021beyond, mei2021understanding}.
\mynewred{It remains an open problem  whether a global optimality convergence result and a low sample complexity can be obtained for the PG with entropy regularization in the stochastic gradient setting.}

In this paper, we provide an affirmative answer to the above question. In particular, we revisit the classical entropy regularized (vanilla) policy gradient method proposed in the seminal work \cite{williams1991function} under the soft-max policy parametrization. We focus on the modern trajectory-level entropy regularization proposed in \cite{haarnoja2017reinforcement}, which is shown to improve over the original one-step entropy regularization adopted in \cite{williams1991function, mnih2016asynchronous} and \cite{o2016combining}.
Our contributions are summarized below: 
\begin{itemize}
    \item We begin by proposing two new entropy regularized stochastic PG estimators. The first one is an unbiased visitation measure-based estimator, whereas the second one is a nearly unbiased yet more practical trajectory-based estimator. {These (nearly) unbiased stochastic PG estimators are the first  likelihood-ratio-based estimators in the literature with a {trajectory-level} entropy regularization.}  We show that although the estimators themselves are unbounded in general due to the entropy-induced logarithmic policy rewards, the variances indeed remain uniformly bounded.
    \item One main challenge on extending the result in \cite{mei2020global} to the stochastic PG setting is the non-coercive landscape\footnote{\myred{A continuous function $f(x)$ that is defined on $\mathbb{R}^n$ is called coercive if $\lim_{\norm{x} \rightarrow \infty} f(x) = + \infty$.}} of the entropy-regularized RL.
    To overcome this challenge, we propose a two-phase stochastic PG algorithm that uses a large batch size in the first phase and uses a small batch size in the second phase.  We establish a global optimality convergence result and a sample complexity of $\widetilde{\mathcal{O}}(\frac{1}{\epsilon^2})$ for the proposed algorithm under the softmax parameterization.  Our result is the first to achieve the sample complexity of $\widetilde{\mathcal{O}}(\frac{1}{\epsilon^2})$  for the stochastic entropy-regularized vanilla PG method and matches the sample complexity of the natural PG \cite{cen2020fast} in terms of dependence on $\epsilon$.
\end{itemize}

\subsection{Related work} \label{sec: related work}


\mynewred{
It has been shown in \cite{ziebart2010modeling} that the entropy-regularized RL formulation provides a substantial improvement in exploration and robustness. An actor-critic method is proposed in \cite{haarnoja2018soft} which updates the policy towards the exponential of the new Q-function and projects the improved policy onto the desired set of policies in the policy improvement step. Instead of using the likelihood ratio gradient estimator \cite{williams1992simple}, the gradient estimator for their policy improvement is based on the re-parameterization technique \cite[Equation (13)]{haarnoja2018soft} where the function approximation is inevitable and the theoretical analysis is challenging.  In contrast, we focus on the stochastic PG method and  classical likelihood ratio estimators.}

\mynewred{
Stochastic PG estimators with the original one-step entropy regularization have been proposed and adopted in \cite{williams1991function,mnih2016asynchronous,o2016combining}. For trajectory-level entropy regularization, an exact (visitation measure-based) PG formula has been derived in \cite{ahmed2019understanding} and later re-derived in the soft-max policy parametrization setting in \cite{mei2020global}, while stochastic PG estimators have not been formally proposed or studied in the literature.  
\cite{schulman2017equivalence} provides a stochastic PG estimator for the value function with a related but different regularization term: trajectory-level KL-divergence regularization. However, KL-divergence regularization is far more aggressive and less used in practice in reinforcement learning compared with the entropy regularization \cite{agarwal2020optimality}).
}

The theoretical understanding of policy-based methods has received considerable attention recently \cite{agarwal2020optimality,lan2021policy,xiao2022convergence,mei2020global,cen2020fast,shani2020adaptive,bhandari2019global,zhang2021convergence,NEURIPS2020_30ee748d,li2021softmax}.  Several techniques have been developed to improve standard PG and achieve a linear convergence rate, such as adding entropy regularization \cite{mei2020global,agarwal2020optimality, lan2021policy, cen2020fast}, exploiting natural geometries based on Bregman divergences  leading to NPG or policy mirror descent \cite{xiao2022convergence, lan2021policy, cen2020fast}, and using a geometry-aware normalized PG (GNPG) approach to exploit the non-uniformity of the value function \cite{mei2021leveraging}.
For the stochastic policy optimization, 
the existing results have mostly focused on policy mirror ascent methods with the goal of reducing the stochastic analysis to the estimation of the Q-value function \cite{cen2020fast, lan2021policy}, as well as incorporating variance reduction techniques to improve the sample complexity of the vanilla PG \cite{liu2020improved, ding2021global}.
The prior literature still lacks globally optimal convergence results and sample complexity for stochastic (vanilla) PG with the entropy regularization.

\subsection{Notation}\label{sec: notation section}
The set of real numbers is shown as $\mathbb R$. ${u} \sim \mathcal{U}$  means that $ u$ is a random vector sampled from the distribution $\mathcal{U}$. 
We use $|\mathcal{X}|$ to denote the cardinality of  a finite set $\mathcal{X}$.
\myred{The notations $\mathbb{E}_{\xi}[\cdot]$ and $\mathbb{E}[\cdot]$ refer to the expectation over the random variable $\xi$ and over all of the randomness. The notation $\text{Var}[\cdot]$ refers to the variance.}
$\Delta(\mathcal{X})$ denotes the probability simplex over a finite set $\mathcal{X}$. 
For vectors ${x}, {y} \in \mathbb{R}^{d}$, let $\|{x}\|_1$, $\|{x}\|_2$ and $\|{x}\|_\infty$ denote the $\ell_{1}$-norm, $\ell_{2}$-norm and $\ell_{\infty}$-norm. We use
$\left\langle {x}, {y}\right\rangle$ to denote the inner product. For a matrix $A$, the notation $A \succcurlyeq0$ means that $A$ is positive semi-definite.
Given a variable $x$, the notation $a=\mathcal{O}(b(x))$ means that $a \leq C \cdot b(x)$ for some constant $C>0$ that is independent of $x$. \myred{Similarly, $a=\widetilde{\mathcal{O}}(b(x))$ indicates that the previous inequality may also depend on the function $\log(x)$, that is, $a \leq C \cdot b(x)\cdot\log(x)$, where $C>0$ is again independent of $x$. 
We use $\text{Geom}(x)$ to denote a geometric distribution with the parameter $x$. Let the notation $\mathbb{1}_{A}$ denote the indicator function of an event $A \subseteq \Omega$, i.e., $\mathbb{1}_{A}(\omega)=1$ if $\omega \in A$ and $\mathbb{1}_{A}(\omega)=0$ otherwise. } \myred{For a given stochastic algorithm, let $\mathcal{F}_t$ denote the  \myred{$\sigma$-}field generated by the history of the algorithm up to the iteration $t$,  just before the randomness at the iteration $t$ is generated. We define $\EE^t\coloneqq \EE[\cdot|\mathcal{F}_t]$ as the expectation operator conditioned on the \myred{$\sigma$-}field $\mathcal{F}_t$.}

\section{Preliminaries}\label{sec:prelim}
\textbf{Markov decision processes.}
RL is generally modeled as a discounted Markov decision process (MDP) defined by a tuple $(\mathcal{S},\mathcal{A}, \mathbb{P},r,\gamma)$. Here, $\mathcal{S}$ and $\mathcal{A}$ denote the finite state and action spaces; $\mathbb{P}(s^\prime|s,a)$ is the probability that the agent transits from the state $s$ to the state $s^\prime$ under the action $a\in \mathcal{A}$; $r(s,a)$ is the reward function, i.e., the agent obtains the reward $r(s_h,a_h)$ after it takes the action $a_h$ at the state $s_h$ at time $h$;
$\gamma\in (0,1)$ is the discount factor. Without loss of generality, we assume that $r(s,a) \in [0,\bar{r}]$ for all $s\in\mathcal{S}$ and $a\in\mathcal{A}$.
The policy $\pi(a|s)$ at the state $s$ is usually represented by a conditional probability distribution $\pi_\theta(a|s)$ associated to the parameter $\theta\in \mathbb{R}^d$. 
Let $\tau^\infty=\{s_0,a_0,s_1,a_1,\ldots\}$ denote the data of a sampled trajectory under policy $\pi_\theta$ with the probability distribution over the trajectory as
$p(\tau^\infty=|\theta,\rho)\coloneq\rho(s_0) \prod_{h=1}^{\infty}\mathbb{P}(s_{h+1}|s_h,a_h)\pi_\theta(a_h|s_h),$
where $\rho \in \Delta(\mathcal{S})$ is the probability distribution of the initial state $s_0$.

\textbf{Value functions and Q-functions.}
Given a policy $\pi$, one can define the state-action value function $Q^\pi: \mathcal{S}\times \mathcal{A} \rightarrow \mathbb{R}$ as 
\begin{align*}
    Q^{\pi}(s,a)\coloneq\EE_{\hspace{-0.3cm}\substack{a_h\sim \pi(\cdot|s_h)\\ s_{h+1}\sim\mathbb{P}(\cdot|s_h,a_h)}}\left[\sum_{h=0}^\infty \gamma^h r(s_h,a_h) \bigg\rvert s_0=s,a_0=a \right].
\end{align*}
The state-value function $V^\pi: \mathcal{S}\rightarrow \mathbb{R}$ and the advantage function $A^\pi: \mathcal{S}\times \mathcal{A} \rightarrow \mathbb{R}$  can be defined as 
$
V^\pi(s)\coloneq \EE_{a\sim\pi(\cdot|s)}[Q^\pi(s,a)],\ A^\pi(s,a)\coloneq Q^\pi(s,a)-V^\pi(s).
$
\myred{The goal is to find an optimal policy in the underlying policy class that maximizes the expected discounted return under the initial state distribution, namely,
$\max_{\theta\in \mathbb{R}^d}  V^{\pi_\theta}(\rho)\coloneq \EE_{s_0 \sim \rho} [V^{\pi_\theta}(s_0)].$
For the notational convenience, we will denote $V^{\pi_\theta}(\rho)$ by the shorthand notation $V^\theta(\rho)$.} 

\textbf{Exploratory initial distribution.}
The discounted state visitation distribution $d_{s_0}^\pi$ is defined as 
$d_{s_0}^\pi(s)\coloneq (1-\gamma) \sum_{h=0}^\infty \gamma^h \mathbb{P}(s_h=s|s_0,\pi),$
where $\mathbb{P}(s_h=s|s_0,\pi)$ is the state visitation probability that $s_h$ is equal to $s$ under the policy $\pi$ starting from the state $s_0$. The discounted state visitation distribution under the initial distribution $\rho$ is defined as $d_\rho^\pi(s)\coloneq\EE_{s_0\sim\rho} [d_{s_0}^\pi(s)]$.
Furthermore, the state-action visitation distribution induced by $\pi$ and the initial state distribution $\rho$ is defined as $v_\rho^{\pi}(s,a)\coloneq d_{\rho}^\pi(s)\pi(a|s)$, which can also be written as 
$v_\rho^{\pi}(s,a) \coloneq  (1-\gamma) \EE_{s_0\sim\rho}\sum_{h=0}^\infty \gamma^h \mathbb{P}(s_h=s,a_h=a|s_0,\pi),$
where $\mathbb{P}(s_h=s,a_h=a|s_0,\pi)$ is the state-action visitation probability that $s_h=s$ and $a_h=a$ under  $\pi$ starting from the state $s_0$. To facilitate the presentation of the main results of the paper, we assume that the state distribution $\rho$ for the performance measure is exploratory \citep{mei2020global, bhandari2019global}, i.e., $\rho(\cdot)$ adequately covers the entire state distribution: 
\begin{assumption}
The state distribution $\rho$ satisfies $\rho(s)>0$ for all $s\in\mathcal{S}$.
\end{assumption}

In practice, when the above assumption is not satisfied, we can optimize under another initial distribution $\mu$, i.e., the gradient is taken with respect to the optimization measure $\mu$, where $\mu$ is usually chosen as an exploratory initial distribution that adequately covers the state distribution of some optimal policy.
It is shown in \cite{agarwal2020optimality} that the difficulty of the exploration problem faced by PG algorithms can be captured through the distribution mismatch coefficient defined as $\norm{\frac{d_\rho^\pi}{\mu}}_\infty$, where $\frac{d_\rho^\pi}{\mu}$ denotes component-wise division.

\textbf{Soft-max policy parameterization.}
In this work, we consider the soft-max parameterization -- a
widely adopted scheme that naturally ensures that the policy lies in the probability simplex. Specifically, for an unconstrained parameter $\theta \in \mathbb{R}^{|\mathcal{S}||\mathcal{A}|}$, $\pi_\theta(a|s)$ is chosen to be 
$\frac{\exp{(\theta_{s,a})}}{\sum_{a^\prime \in \mathcal{A}}\exp{(\theta_{s,a^\prime}})}.$
The soft-max parameterization is generally used for MDPs with finite state and action spaces. It is complete in the sense that every stochastic policy can be represented by this class. For the soft-max parameterization, 
it can be shown that the gradient and Hessian of the function $\log \pi_\theta(a|s)$ are bounded, i.e., for all $\theta \in \mathbb{R}^{|\mathcal{S}||\mathcal{A}|}$, $s\in\mathcal{S}$ and $a\in\mathcal{A}$, we have: $\norm{\nabla \log\pi_\theta(a|s)}_2\leq 2,\ \norm{\nabla^2 \log\pi_\theta(a|s)}_2\leq 1.$


\textbf{RL with entropy regularization.}
Entropy is a commonly used regularization in RL to promote exploration and discourage premature convergence to suboptimal policies \citep{haarnoja2017reinforcement, schulman2017equivalence, eysenbach2019if}.
It is far less aggressive in penalizing small probabilities, in comparison to other common regularizations such as log barrier functions \citep{agarwal2020optimality}.
In the entropy-regularized RL (also known as maximum entropy RL), near-deterministic policies are penalized, which is achieved by modifying the value function to
\begin{align} \label{eq: entropy regularized objective}
   V^\pi_
   \lambda(\rho)=V^\pi(\rho)+\lambda \mathbb{H}(\rho,\pi),
\end{align}
where $\lambda \geq 0$ determines the strength of the penalty and $\mathbb{H}(\rho,\pi)$ stands for the discounted entropy defined as 
$$\mathbb{H}(\rho,\pi)\coloneqq \EE_{\substack{s_0\sim\rho,a_t\sim \pi(\cdot|s_t)\\ s_{t+1}\sim \mathbb{P}(\cdot|s_t,a_t)}} \left[ \sum_{t=0}^\infty -\gamma^t \log \pi(a_t|s_t)\right] .  
$$
Equivalently, $V^\pi_
   \lambda(\rho)$ can be viewed as the weighted value function of $\pi$ by adjusting the instantaneous reward to be policy-dependent regularized version as
$r^\lambda(s,a)\coloneqq r(s,a)-\lambda \log \pi(a|s)$,  for all $(s,a)\in \mathcal{S} \times \mathcal{A}.$
We also define $V_{\lambda}^{\pi}(s)$ analogously when the initial state is fixed at a given state $s \in \mathcal{S}$. The regularized Q-function $Q_{\lambda}^{\pi}$ of a policy $\pi$, also known as the soft $Q$-function, is related to $V_{\lambda}^{\pi}$ as (for every $s\in\mathcal{S}$ and $a\in\mathcal{A}$)
\begin{align*}
\begin{aligned}
Q_{\lambda}^{\pi}(s, a) &=r(s, a)+\gamma \mathbb{E}_{s^{\prime} \sim P(\cdot \mid s, a)}\left[V_{\lambda}^{\pi}\left(s^{\prime}\right)\right], \\
 V_{\lambda}^{\pi}(s) &=\mathbb{E}_{a \sim \pi(\cdot \mid s)}\left[-\lambda \log \pi(a \mid s)+Q_{\lambda}^{\pi}(s, a)\right].
\end{aligned}
\end{align*}

\textbf{Bias due to entropy regularization.}
Due to the presence of regularization, the optimal solution will be biased with the bias disappearing as $\lambda \rightarrow 0$. More precisely, the optimal policy $\pi^\ast_\lambda$ of the entropy-regularized problem could also be nearly optimal in terms of the unregularized objective function, as long as the
regularization parameter $\lambda$ is chosen to be small.
Denote by $\pi^\ast$ and $\pi^\ast_\lambda$ the policies that maximize the objective function and the entropy-regularized objective function with the regularization parameter $\lambda$, respectively. Let $V^\ast$ and $V^\ast_\lambda$ represent the resulting optimal objective value function and the optimal regularized objective value function. 
\cite{cen2020fast} shows a simple but crucial connection between $\pi^\ast$ and $\pi^\ast_\lambda$
via the following sandwich bound:
$$
V^{\pi^\ast_\lambda}(\rho)\leq   V^{\pi^\ast} (\rho)    \leq  V^{\pi^\ast_\lambda}(\rho)+\frac{\lambda \log{|\mathcal{A}|}}{1-\gamma},
$$
which holds for all initial distribution $\rho$. 

\section{Stochastic PG methods for entropy regularized RL}
\subsection{Review: Exact PG methods}
The PG method is one of the most popular approaches for a direct policy search in RL \citep{sutton2018reinforcement}.  The vanilla PG with exact gradient information and the entropy regularization is summarized in Algorithm \ref{alg:Policy Gradient method}.
\begin{algorithm}[H] 
\caption{Exact PG method}
\label{alg:Policy Gradient method}
\begin{algorithmic}[1]
\STATE \textbf{Inputs}: $\{\eta_t\}_{t=1}^T$, $\theta_1$.
\FOR{$t = 1, 2, \dots, T-1$}
\STATE $\theta_{t+1}=\theta_{t}+\eta_t \nabla V_\lambda^{\theta_t}(\rho)$.
\ENDFOR{}
\STATE \textbf{Outputs}: $\theta_{T}$.
\end{algorithmic}
\end{algorithm}
The uniform boundedness of the reward function $r$ implies that the absolute value of the entropy-regularized state-value function and Q-value function are bounded.
\begin{lemma}[\cite{mei2020global}] \label{lemma: bound of V and Q}
${V}^{{\theta}}_\lambda(s) \leq \frac{\bar{r}+\lambda \log |\mathcal{A}|}{1-\gamma}$ and $Q_{\lambda}^{\pi}(s, a)\leq \frac{\bar{r}+\lambda \log |\mathcal{A}|}{1-\gamma}$ for all $(s,a) \in \mathcal{S}\times \mathcal{A}$ and $\theta\in {\mathbb{R}^{|\mathcal{S}| |\mathcal{A}|}}$.
\end{lemma}

Under the soft-max policy parameterization, one can obtain the following expression for the gradient of $V_{\lambda}^{\pi}(s)$ with respect to the policy parameter $\theta$:

\begin{lemma}[Proposition 2 in \cite{cayci2021linear}] \label{lemma: Max Ent PG}
The entropy regularized PG with respect to $\theta$ is
\begin{align}\label{eq: Maxent PG}
&\nabla V^{\theta}_\lambda(\rho)=\\ 
\nonumber&\frac{1}{1-\gamma} \EE_{s,a\sim v_{\rho}^{\pi_{\theta}}} \left[ \nabla_\theta \log \pi_\theta (a|s) \left(Q^{\theta}_\lambda(s, a)-\lambda \log \pi_{\theta}(a \mid s) \right) \right],
\end{align}
where
$$
\frac{\partial \log \pi_\theta(a|s)}{\partial {\theta_{s^\prime,a^\prime}}}=
\begin{cases} 
-\pi_\theta(a^\prime|s^\prime) \pi_\theta(a|s), & (s^\prime,a^\prime)\neq(s,a),\\
\pi_\theta(a|s) -\pi_\theta(a|s) \pi_\theta(a|s), & (s^\prime,a^\prime)=(s,a).
\end{cases}
$$
Furthermore, the entropy regularized PG is bounded, i.e., $\norm{\nabla V^{\theta}_\lambda(\rho)}\leq G$ for all $\rho \in \Delta(\mathcal{S})$ and $\theta\in {\mathbb{R}^{|\mathcal{S}| |\mathcal{A}|}}$, where $G\coloneq \frac{2(\bar{r}+\lambda \log |\mathcal{A}|)}{(1-\gamma)^2}$.
\end{lemma}

In addition, it is shown  that the PG $\nabla V^{\theta}_\lambda(\rho)$ is Lipschitz continuous.
\begin{lemma}[Lemmas 7 and 14 in \cite{mei2020global}]\label{lemma: Lipschitz-Continuity of Policy Gradient}
 The PG $\nabla V^{\theta}_\lambda(\rho)$ is Lipschitz continuous with some constant $L>0$, i.e.,
$\left\|\nabla V^{\theta_1}_\lambda(\rho)-\nabla V^{\theta_2}_\lambda(\rho)\right\| \leq L \cdot\left\|\theta^{1}-\theta^{2}\right\|$,
for all $\theta^{1}, \theta^{2} \in \mathbb{R}^{d}$, where the value of the Lipschitz constant $L$ is defined as
$L:= \frac{8\bar{r}+\lambda(4+8\log{|\mathcal{A}|})}{(1-\gamma)^3}$.
\end{lemma}

\textbf{Challenges for designing entropy regularized PG estimators.} 
Existing works either consider one-step entropy regularization \citep{williams1992simple, mnih2016asynchronous}, KL divergence \citep{schulman2017equivalence}, or the re-parametrization technique  \citep{haarnoja2017reinforcement,haarnoja2018soft} (which introduces  approximation errors that are difficult to quantify exactly).  
In general, the regularized reward $r-\lambda \log{\pi_\theta}$ is policy-dependent and unbounded even though the original reward $r$ is uniformly bounded. Hence, the existing estimators for the un-regularized setting must be modified to account for the policy-dependency and unboundedness while maintaining the essential properties of (nearly) unbiasedness and bounded variances. In the subsequent sections, we propose two (nearly) unbiased estimators and show that although the estimators may be unbounded due to unbounded regularized rewards, the variances are indeed bounded. The proofs of the results in this section can be found in Section \ref{sec: appex_SPG} of the supplemental materials.

\subsection{Sampling the unbiased PG}
It results from \eqref{eq: Maxent PG} that in order to obtain an unbiased sample of $\nabla V^{\theta}_\lambda(\rho)$, we need to first draw a state-action pair $(s, a)$ from the distribution $\nu_\rho^{\pi_\theta}(\cdot,\cdot)$ and then obtain an unbiased estimate of the action-value function $Q^{\theta}_\lambda(s, a)$.
For the standard discounted infinite-horizon RL setting with bounded reward functions, \cite{zhang2020global} proposes an unbiased estimate of the PG using the random horizon with a geometric distribution and the Monte-Carlo rollouts of finite horizons. However, their result cannot be immediately applied to the entropy-regularized RL setting since the entropy-regularized instantaneous reward $r(s,a)-\lambda \log \pi(a|s)$ could be unbounded when $\pi(a|s) \rightarrow 0$. Fortunately, we can still show that an unbiased PG estimator with the bounded variance for the entropy regularized RL can be obtained in a similar fashion as in \cite{zhang2020global}. In particular, we will use a random horizon that follows a certain geometric distribution in the sampling process. \myred{To ensure that $\left(s_{H}, a_{H}\right) \sim \nu_\rho^{\pi_\theta}(s,a)$, we will use the last sample $\left(s_{H}, a_{H}\right)$ of a finite sample trajectory $\left(s_{0}, a_{0}, s_{1}, a_1, \ldots, s_{H}, a_{H}\right)$ to be the sample at which $Q^{\theta}_\lambda(\cdot, \cdot)$ is evaluated, where the horizon $H \sim$ $\text{Geom}(1-\gamma)$.}
Moreover, given $\left(s_{H}, a_{H}\right)$, we will perform Monte-Carlo rollouts for  another trajectory with the horizon $H^{\prime} \sim \text{Geom}\left(1-\gamma^{1 / 2}\right)$ independent of $H$, and estimate the advantage function value $Q^{\theta}_\lambda(s, a)$ along the trajectory $(s_0^\prime, a_0^\prime, \ldots, s_{H^\prime}^\prime)$ with $s_{0}^{\prime}=s, a_{0}^{\prime}=a$ as follows:
\begin{align} \label{eq: Q estimator}
\hat{Q}^{{\theta}}_\lambda(s, a)=&r\left(s_{0}^{\prime}, a_{0}^{\prime}\right)+\sum_{t=1}^{H^{\prime}} \gamma^{t / 2} \cdot\left( r\left(s_{t}^{\prime}, a_{t}^{\prime}\right)-\lambda \log{\pi_\theta(a^{\prime}|s^{\prime})}\right).
\end{align}

The subroutines of sampling one pair $(s,a)$ from $\nu_\rho^{\pi_\theta}(\cdot,\cdot)$, estimating $\hat{Q}^{{\theta}}_\lambda(s, a)$, and estimating $\hat{V}^{{\theta}}_\lambda(s)$ are summarized as \textbf{Sam-SA} and \textbf{Est-EntQ} in Algorithms  \ref{alg:SamSA} and \ref{alg:EstQ}, respectively.

\begin{algorithm}[ht!] 
\caption{\textbf{Sam-SA}: Sample for $s,a \sim \nu_\rho^{\pi_\theta}(\cdot,\cdot)$}
\label{alg:SamSA}
\begin{algorithmic}[1]
\STATE \textbf{Inputs}: $\rho$, $\theta, \gamma$. 
\STATE Draw $H\sim\text{Geom}(1-\gamma)$.
\STATE Draw $s_0\sim \rho$ and $a_0\sim \pi_{\theta}(\cdot|s_0)$
\FOR{$h = 1, 2, \dots,H-1$}
\STATE Simulate the next state $s_{h+1} \sim \mathbb{P}(\cdot|s_h,a_h)$ and action $a_{h+1} \sim \pi_{\theta_t}(\cdot|s_{h+1})$.
\ENDFOR{}
\STATE \textbf{Outputs}: $s_H, a_H$.
\end{algorithmic}
\end{algorithm}

\begin{algorithm}[ht!] 
\caption{\textbf{Est-EntQ}: Unbiasedly estimating entropy-regularized Q function}
\label{alg:EstQ}
\begin{algorithmic}[1]
\STATE \textbf{Inputs}: $s, a, \gamma, \lambda$ and $\theta$. 
\STATE Initialize $s_0 \leftarrow s, a_0 \leftarrow a, \hat{Q}\leftarrow r(s_0,a_0)$.
\STATE Draw $H'\sim\text{Geom}(1-\gamma^{1/2})$.
\FOR{$h = 0, 1, \dots,H'-1$}
\STATE Simulate the next state $s_{h+1} \sim \mathbb{P}(\cdot|s_h,a_h)$ and action $a_{h+1} \sim \pi_{\theta}(\cdot|s_{h+1})$.
\STATE Collect the instantaneous reward $r\left(s_{h+1}, a_{h+1}\right)-\lambda \log{\pi_\theta(a_{h+1}|s_{h+1})}$ and add to the value  $\hat{Q}$: $\hat{Q} \leftarrow \hat{Q} +\gamma^{(h+1)/2} \left(r\left(s_{h+1}, a_{h+1}\right)-\lambda \log{\pi_\theta(a_{h+1}|s_{h+1})}\right)$,
\ENDFOR{}
\STATE \textbf{Outputs}: $\hat{Q}$.
\end{algorithmic}
\end{algorithm}


Motivated by the form of PG in \eqref{eq: Maxent PG}, we propose the following stochastic estimator:
\begin{align} \label{eq: SPG}
&\hat{\nabla} V^{\theta}_\lambda(\rho)=\\
\nonumber &\frac{1}{1-\gamma} \nabla_\theta \log \pi_\theta (a_H|s_H) \left(\hat{Q}^{\theta}_\lambda(s_H, a_H)-\lambda \log \pi_{\theta}(a_H \mid s_H) \right),
\end{align}
where $s_H,a_H \leftarrow \textbf{Sam-SA}(\rho, \theta, \gamma)$ and $\hat{Q}^{\theta}_\lambda$ is defined in \eqref{eq: Q estimator}.
The following lemma shows that the stochastic PG  \eqref{eq: SPG} is an unbiased estimator of ${\nabla} V^{\theta}_\lambda(\rho)$.

\begin{lemma} \label{lemma: unbias}
For $\hat{\nabla} V^{\theta}_\lambda(\rho)$ defined in \eqref{eq: SPG}, we have 
$\EE[\hat{\nabla} V^{\theta}_\lambda(\rho)]={\nabla} V^{\theta}_\lambda(\rho).$
\end{lemma}

The next lemma shows that the proposed PG estimator $\hat{\nabla} V^{\theta}_\lambda(\rho)$ has a bounded variance even if it is unbounded when $\pi_\theta$ approaches a deterministic policy.

\begin{lemma}\label{lemma: bounded variance}
\mynewred{For $\hat{\nabla} V^{\theta}_\lambda(\rho)$ defined in \eqref{eq: SPG}, we have 
$
\text{Var}[\hat{\nabla} V^{\theta}_\lambda(\rho)]\leq \sigma^2,
$
where $\sigma^2=\frac{8}{(1-\gamma)^2}\left( \frac{\bar{r}^2+(\lambda\log |\mathcal{A}|)^2}{(1-\gamma^{1/2})^2}\right)$ and $\bar{r}$ is the upper bound of the reward}.
\end{lemma}

\subsection{Sampling the trajectory-based PG}
Compared to the unbiased PG with a random horizon in \eqref{eq: SPG}, a more practical PG estimator is the trajectory-based PG. To derive the trajectory-based PG for the entropy-regularized RL, we first notice that the gradient ${\nabla} V^{\theta}_\lambda(\rho)$ can also be written as
\begin{align*}
 &{\nabla} V^{\theta}_\lambda(\rho)=\\
\nonumber &\EE  \left[\left(\sum_{t=0}^\infty \nabla \log \pi_\theta(a_t|s_t) \right)\left(\sum_{t=0}^\infty \gamma^t \left(r(s_t,a_t) -\lambda \log \pi_\theta(a_t|s_t) \right) \right)\right], 
\end{align*}
where the expectation is taken over the trajectory distribution, i.e., $\tau^\infty=\sim p(\tau^\infty=|\theta)$.


Since the distribution $p(\tau^\infty=|\theta)$ is unknown, ${\nabla} V^{\theta}_\lambda(\rho)$ needs to be estimated from samples.   
The trajectory-based estimators
include REINFORCE \citep{williams1992simple}, PGT \citep{sutton1999policy}  and GPOMDP \citep{baxter2001infinite}. 
In practice, the truncated versions of these trajectory-based PG estimators are used to approximate the infinite sum in the PG estimator.  Let $\tau^H=\{s_0,a_0,s_1,\ldots, s_{H-1}, a_{H-1}, s_H\}$ denote the truncation of the full trajectory $\tau^\infty=$ of length $H$.
Then,
with the commonly used truncated  GPOMDP, the truncated PG estimator for ${\nabla} V^{\theta}_\lambda$ can be written as:
\begin{align} \label{eq: PGT estimator with entropy}
&\hat{\nabla} V^{\theta, H}_\lambda(\rho)=\\
\nonumber &\sum_{h=0}^{H-1}\left(\sum_{j=0}^{h} \nabla \log \pi_\theta(a_j|s_j)\right) \gamma^h \left( r(s_h,a_h) -\lambda \log \pi_\theta(a_h|s_h) \right).
\end{align}

Due to the horizon truncation, the PG estimator \eqref{eq: PGT estimator with entropy} may no longer be unbiased, but its bias can be very small with a large horizon $H$.
\begin{lemma} \label{lemma:bias of truncation}
For $\hat{\nabla} V^{\theta, H}_\lambda(\rho)$ defined in \eqref{eq: PGT estimator with entropy}, we have 
\begin{align*}
 \norm{\EE[\hat{\nabla} V^{\theta, H}_\lambda(\rho)]-{\nabla} V^{\theta}_\lambda(\rho)}_2 \leq \frac{ 2(\bar{r}+\lambda\log{|\mathcal{A}|})\gamma^H}{(1-\gamma)} \left(H+\frac{1}{1-\gamma}\right).
\end{align*}
\end{lemma}
From Lemma \ref{lemma:bias of truncation}, we can observe that the bias is proportional to $\gamma^H$ and thus can be controlled to be arbitrarily  small with a constant horizon up to some logarithmic term.
We then show that the truncated PG estimator $\hat{\nabla} V^{\theta, H}$ has a bounded variance even if it may be unbounded when $\pi_\theta$ approaches a deterministic policy.

\begin{lemma} \label{lemma:bounded variance of truncation}
For $\hat{\nabla} V^{\theta, H}_\lambda(\rho)$ defined in \eqref{eq: PGT estimator with entropy}, we have 
$$\text{Var}(\hat{\nabla} V^{\theta, H}_\lambda(\rho))\leq
\frac{12\bar{r}^2+24 \lambda^2(\log |\mathcal{A}|)^2}{(1-\gamma)^4}.$$
\end{lemma}


\subsection{Batched PG algorithms}
In practice, we can sample and compute a batch of independently and identically  distributed PG estimators $\{\hat{\nabla} V^{\theta,i}_\lambda(\rho)\}_{i=1}^B$ 
where $B$ is the batch size, in order to reduce the estimation variance.
To maximize the entropy-regularized objective function \eqref{eq: entropy regularized objective}, we can then update the policy parameter $\theta$ by iteratively running gradient-ascent-based algorithms, i.e., $\theta_{t+1}=\theta_t+ \frac{\eta_t}{B} \sum_{i=1}^{B} \hat{\nabla} V^{\theta,i}_\lambda(\rho),$
where $\eta_t>0$ is the step size. The details of the unbiased PG algorithm with a random horizon 
for the entropy-regularized RL are provided in Algorithm \ref{alg:unbiased PG}.

\begin{algorithm}[H] 
\caption{\textbf{Ent-RPG}: Random-horizon PG for Entropy-regularized RL}
\label{alg:unbiased PG}
\begin{algorithmic}[1]
\STATE \textbf{Inputs}: $\rho, \lambda$, $\theta_1, B, T, \{\eta_t\}_{t=1}^T$. 
\FOR{$t = 1, 2, \dots,T$}
\FOR{$i = 1, 2, \dots,B$}
\STATE $s_{H_t}^i, a_{H_t}^i \leftarrow \textbf{SamSA}(\rho, \theta_t, \gamma)$.
\STATE $\hat{Q}^{\theta_t,i}_\lambda\leftarrow \textbf{Est-EntQ}(s_{H_t}^i, a_{H_t}^i,\theta_t,\gamma, \lambda)$.

\ENDFOR{}
\STATE $\theta_{t+1}\leftarrow \theta_t + \frac{\eta_t}{(1-\gamma)B} \sum_{i=1}^B \left[\nabla_\theta \log \pi_{\theta_t} (a_{H_t}^i|s_{H_t}^i) \right.$
$\hspace{7cm} \left. \left( \hat{Q}^{\theta_t,i}_\lambda-\lambda \log \pi_{\theta_t}(s_{H_t}^i\mid a_{H_t}^i)\right)\right]$
\ENDFOR{}
\STATE \textbf{Outputs}: $\theta_T$.
\end{algorithmic}
\end{algorithm}


\begin{remark}
For the simplicity of the presentation, we focus on deriving the stochastic PG estimator for the soft-max policy parameterization. However, our results in this section {(and also the stationary point convergence result in Section \ref{conv_stationary_point} below)} can be easily extended to the general parameterization $\pi_\theta$ as long as   $\norm{\nabla \log\pi_\theta(a|s)}_2$ and  $\norm{\nabla^2 \log\pi_\theta(a|s)}_2$ are bounded for all $(s,a)\in\mathcal{S} \times \mathcal{A}$.
\end{remark}

Due to space restrictions and in order to facilitate the presentation of the main ideas, we will mainly focus on the analysis of the unbiased PG estimator in \eqref{eq: SPG} for the rest of the paper. Similar results hold for the trajectory-based PG estimator in \eqref{eq: PGT estimator with entropy} since its bias is exponentially small with respect to the horizon (see Lemma \ref{lemma:bias of truncation}). \myred{The proofs of the results of this section can be found in Appendix \ref{sec: appex_SPG}}. We leave the formal discussion of these results as future work.
\section{Non-coercive landscape} \label{sec:Global convergence of maximum entropy RL}
In this section, we first review some key results for the entropy-regularized RL with the exact PG and highlight the difficulty of generalizing these results to the  stochastic PG setting, due to the non-coercive landscape. 


\subsection{Review: Linear convergence with exact PG}
A key result from \cite{mei2020global} shows that, under the soft-max parameterization, the entropy-regularized value function $V_\lambda^{\theta}(\rho)$ in \eqref{eq: entropy regularized objective} satisfies a non-uniform \L{ojasiewicz} inequality as follows:

\begin{lemma}[Lemma 15 in \cite{mei2020global}] \label{lemma: non-uniform lojasiewicz}
It holds that
$$\norm{ \nabla V_\lambda^{\theta}(\rho)}_2^2\geq C(\theta) (V_\lambda^{\theta^\ast}(\rho)-V_\lambda^{\theta}(\rho)),$$
where $$C(\theta)=\frac{2\lambda}{|\mathcal{S}|} \min_{s} \rho(s) \min_{s,a} \pi_\theta(a|s)^2 \norm{\frac{d_{\rho}^{\pi_\lambda^\ast}}{\rho}}_\infty^{-1}.$$
\end{lemma}

Furthermore,
it is shown in  \cite{mei2020global} that the action probabilities under the soft-max parameterization are uniformly bounded away from zero if the exact PG is available.

\begin{lemma}[Lemma 16 in \cite{mei2020global}] \label{lemma: lower bound of pi}
Using the exact PG (Algorithm \ref{alg:Policy Gradient method}) with $\eta_t{=\eta}\leq \frac{2}{L}$ for the entropy regularized objective, it holds that $\inf_{t\geq 1} \min_{s,a} \pi_{\theta_t}(a|s) >0$.  
\end{lemma}

{
\begin{remark}\label{c-theta1-eta}
Note that by Algorithm \ref{alg:Policy Gradient method}, $\inf_{t\geq 1} \min_{s,a} \pi_{\theta_t}(a|s)$ is only dependent on the initialization $\theta_1$ and step-size $\eta$ (apart from problem dependent constants). Hence hereafter we denote $c_{\theta_1,\eta}=\inf_{t\geq 1} \min_{s,a} \pi_{\theta_t}(a|s)$.  
\end{remark}
}

With Lemmas \ref{lemma: Lipschitz-Continuity of Policy Gradient}, \ref{lemma: non-uniform lojasiewicz} and \ref{lemma: lower bound of pi}, it is shown in Theorem 6 of \cite{mei2020global} that the convergence rate for the entropy regularized PG is $O\left(e^{-Ct}\right)$,
where the value of $C$ depends on $\inf_{t\geq 1} \min_{s,a} \pi_{\theta_t}(a|s)>0$ and $\{\theta_t\}_{t=1}^\infty$ is generated by Algorithm \ref{alg:Policy Gradient method}. With a bad initialization $\theta_1$, $\min_{s,a} \pi_{\theta_1}(a|s)$ could be very small and result in a slow convergence rate. When studying the stochastic PG, this issue of bad initialization will create  more severe challenges on the convergence, which we will discuss in the following sections.



{One main challenge is the boundedness of iterations under the stochastic PG.} 
The iterates of stochastic gradient methods may indeed escape to infinity in general, {rendering the
entire scheme of stochastic approximation useless} \citep{benaim1999dynamics, borkar2009stochastic}. In particular, when using the stochastic truncated PG for the entropy regularized RL, the key result of Lemma \ref{lemma: lower bound of pi} may no longer hold true. This in turn results in the loss of gradient domination condition in guaranteeing the global convergence. 


\subsection{Landscape of a simple bandit example}
To have a better understanding of the landscape of the entropy-regularized value function, we visualize its landscape in this section. For the simplicity of the visualization, we use a simple bandit example (corresponding to $\gamma=0$) with 2 actions, 2 parameters $(\theta_1, \theta_2)$, the reward vector $r=[2,1]$ and the regularization parameter $\lambda=1$. Then, the entropy-regularized value function can be written as
$\pi_{\theta}^{\top}\left(r-\log \pi_{\theta}\right)$.

\begin{figure}[ht]
\centering
\includegraphics[width=7cm]{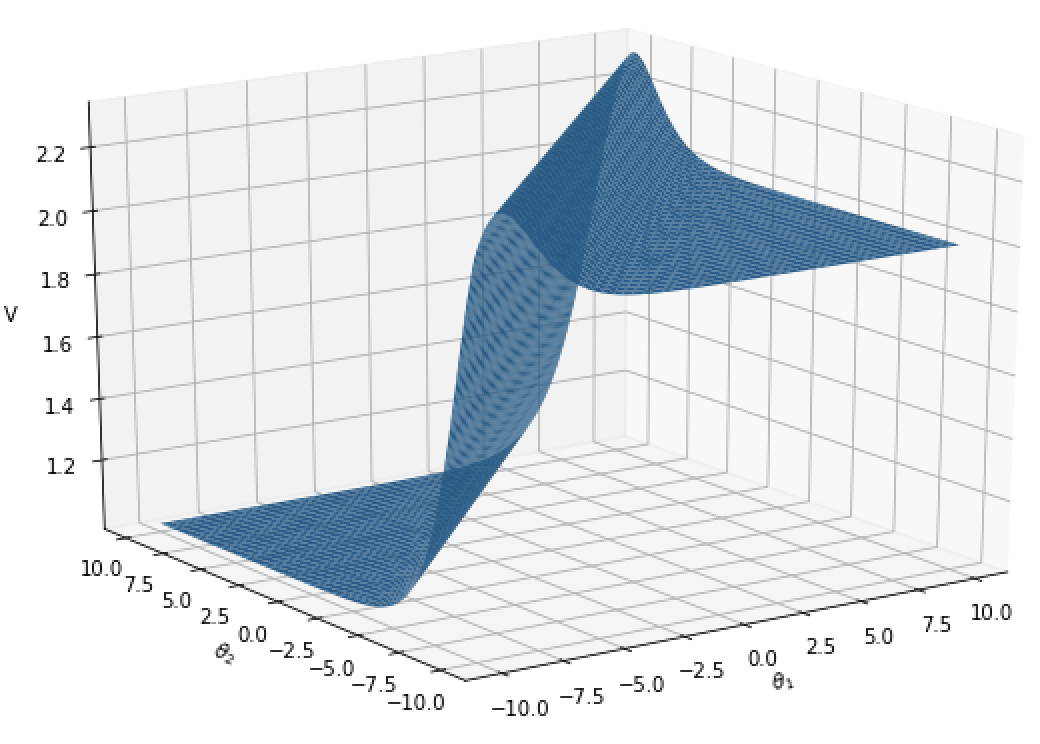}
\caption{Landscape of $\pi_{\theta}^{\top}\left(r-\log \pi_{\theta}\right)$.}
\label{fig: EntRL}
\end{figure}

As shown in Figure \ref{fig: EntRL}, the entropy-regularized value function is not coercive. 
When $\theta_1$ goes to positive (negative) infinity and $\theta_2$ goes to negative (positive) infinity, the landscape will become highly flat. It can also be seen that there is a line space for $(\theta_1, \theta_2)$ at which the entropy-regularized value function is maximum.

When the stochastic PG is used, the search direction may be dominated by the gradient estimation noise at the region where the landscape is highly flat. This may further lead to the failure of the globally optimal convergence for the stochastic PG algorithm if the initial point is at the flat region. 

\subsection{Convergence to the first-order stationary point}\label{conv_stationary_point}
Before presenting our main result, 
we first show that the stochastic PG proposed in Algorithm \ref{alg:unbiased PG} asymptotically converges to a region where the PG vanishes almost surely if a specific adaptive step-size sequence is used.

 \begin{lemma}\label{thm: asym global convergence}
Suppose that the sequence $\{\theta_t\}_{t=1}^\infty$ is generated by Algorithm \ref{alg:unbiased PG} for the entropy regularized objective with the step-sizes satisfying $\sum_{t=1}^\infty \eta_t =\infty, \sum_{t=1}^\infty \eta_t^2 <\infty$ and $\eta_t\leq \frac{2}{L}$ for all $t=1,2,\ldots$. It holds that $\lim_{t\rightarrow \infty} \norm{ \nabla V_\lambda^{\theta_t}(\rho)}_2 =0$ with probability $1$.
 \end{lemma}
 
This result follows from classic results for the Robbins-Monro algorithm \citep{bertsekas2000gradient, benaim1996asymptotic, kushner2012stochastic} when an unbiased PG estimator with the bounded variance, as in Algorithm \ref{alg:unbiased PG}, is used in the update rule. No requirement on the batch size $B$ is needed in Lemma \ref{thm: asym global convergence}.
We now provide the proof of Lemma \ref{thm: asym global convergence} below.


 \begin{proof}


 To prove Lemma \ref{thm: asym global convergence}, it suffices to check the conditions in Proposition 3 of \cite{bertsekas2000gradient} for the objective function $V_\lambda^\theta(\rho)$ and the update rule $\theta_{t+1}=\theta_t+\eta_t (u_t+w_t)$, where $u_t={\nabla}V_\lambda^{\theta_t}(\rho)$ and $w_t=\hat{\nabla}V_\lambda^{\theta_t}(\rho)-{\nabla}V_\lambda^{\theta_t}(\rho)$.
 
\myred{
\begin{enumerate}
 \item From Lemma \ref{lemma: Lipschitz-Continuity of Policy Gradient}, we know that Condition 1 in Proposition 3 of \cite{bertsekas2000gradient} is satisfied with $L= \frac{8\bar{r}+\lambda(4+8\log{|\mathcal{A}|})}{(1-\gamma)^3}$.
     \item Condition 1 in Proposition 3 of \cite{bertsekas2000gradient} is satisfied by the definition of $\theta_t$ and ${\nabla}V_\lambda^\theta(\rho)$.
     \item Condition 1 in Proposition 3 of \cite{bertsekas2000gradient} is satisfied with $c_1=1$ and $c_2=1$.
    \item From Lemma \ref{lemma: unbias} and \ref{lemma: bounded variance}, we know that Condition 4 in Proposition 3 of \cite{bertsekas2000gradient} is satisfied with $A=\frac{8}{(1-\gamma)^2}\left( \frac{\bar{r}^2+(\lambda\log |\mathcal{A}|)^2}{(1-\gamma^{1/2})^2}\right)$.
     \item Condition 1 in Proposition 3 of \cite{bertsekas2000gradient} is satisfied by the definition of $\eta_t$.
 \end{enumerate}}
 \myred{In addition, it results from Lemma \ref{lemma: bound of V and Q} that the entropy-regularized value function $V_\lambda^{\theta}(\rho)$ is bounded.} Thus, by Proposition 3 of \cite{bertsekas2000gradient}, we must have $\lim_{t\rightarrow \infty} \nabla V_\lambda^{\theta_t}(\rho)=0$ with probability $1$. This completes the proof.

 \end{proof}
 

However, since the entropy-regularized value function $V_\lambda^{\theta}(\rho)$ is not coercive in $\theta$ and it may be the case that the gradient $\nabla V_\lambda^{\theta_t}(\rho)$ diminishing to $0$ corresponds to $\theta_t$ going to infinity instead of converging to a stationary point. In addition, the existing results \citep{benaim1999dynamics, benaim1996asymptotic, kushner2012stochastic} on the almost surely stationary point convergence rely on the assumption that the trajectories of the process are bounded, i.e., 
$\sup_{t\geq 0} \norm{\theta_t} <\infty,$  almost surely.
This assumption is proven to hold when the function is coercive \citep{mertikopoulos2020almost}. However, when the function is not coercive, as in our problem, it is very challenging to 
characterize the trade-off between the gradient information and the estimation error without additional assumptions. 

\section{Main result}
To overcome the non-coercive landscape challenge, we propose a two-phase stochastic PG algorithm (Algorithm \ref{alg:2 phases}). In the first phase, we will use a {large} batch size to control the estimation error to guarantee that the stochastic PG is informative even in the regime where the landscape is almost flat. After a certain number of iterations, which is a constant with respect to the optimality gap $\epsilon$, the iteration will reach a region where the landscape has enough curvature information. Then, in the second phase, a {small} batch size is enough to guarantee a fast convergence to the optimal policy.

Before presenting the main result, we first introduce some helpful definitions. 
Let $D(\theta_{t}) =V^{\theta^\ast}_\lambda(\rho)-V^{\theta_t}_\lambda(\rho) $ denote the sub-optimality gap.
Since the optimal policy of \eqref{eq: entropy regularized objective} is unique \cite{cen2020fast}, there must exist a continuum of optimal solutions 
$$\Theta^\ast\coloneqq \{ \theta^\ast \in \mathbb{R}^{|\mathcal{S}| |\mathcal{A}|}: \frac{\exp(\theta_{s,a}^\ast)}{\sum_{a'} \exp(\theta_{s,a'}^\ast)} = \pi_\lambda^\ast (a\mid s), \forall s\in\mathcal{S}, a \in \mathcal{A}\}.$$
In addition, we use $\pi_{\theta^\star}$ and $\pi_{\lambda}^\star$ interchangeably to denote the optimal policy of the entropy-regularized RL.
Let $\{\Bar{\theta}_t\}_{t=1}^T$ denote
the iterates of the algorithm with the exact PG (Algorithm \ref{alg:Policy Gradient method}) with $\eta_t{=\eta} \leq \frac{1}{2L } $ starting from the initial point ${\theta}_1$.
For the soft-max parameterization, we have $\theta_{s, a}=\log \pi_{\theta}(a\mid s) +C_s$ for all $(s,a) \in \mathcal{S} \times \mathcal{A}$, where $\{C_s\}_{s=1}^{|\mathcal{S}|} $ are some constants.
Then,  we have
\begin{align*}
\min_{\theta^\ast \in \Theta^\ast } \norm{\Bar{\theta}_t-{\theta}^\ast}_2 = \norm{\log \pi_{\Bar{\theta}_t} - \log \pi_\lambda^\ast}_2  ,\quad \text{ for all } t=\{1,2, \ldots\}.
\end{align*}
Furthermore, by  Lemma \ref{lemma: lower bound of pi}, {we can  define 
$
\bar{\Delta}\coloneqq \|\log c_{\bar{\theta}_1,\eta} - \log \pi_{\lambda}^\star\|_2,
$
where $c_{\bar{\theta}_1,\eta}>0$ is defined in Remark \ref{c-theta1-eta}. Note that $\bar{\Delta}$ is only dependent on $\bar{\theta}_1$ and $\eta$ (apart from problem dependent constants), and
}
$\norm{\log \pi_{\bar{\theta}_t} - \log \pi_\lambda^\ast}_2 $ $ \leq \bar{\Delta}$ for any $\bar{\theta}_1$ and $\eta\leq \frac{1}{2L}$.
In addition, with a fair degree of hindsight and for some $\delta>0$, we define the stopping time for the iterates $\{\theta_t\}_{t=1}^T$ as
\begin{align}\label{eq: stopping time}
\tau\coloneqq \min\left\{t \Big|\min_{\theta^\ast \in \Theta^\ast } \norm{\theta_t-\theta^\ast}_2  >\left(1+\frac{1}{\delta}\right)\bar{\Delta} \right\},    
\end{align}
which is the index of the first iterate that exits the bounded region $$\mathcal{G}^0_\delta\coloneqq \left\{\theta \in \mathbb{R}^{|\mathcal{S}| |\mathcal{A}|}: \min_{\theta^\ast \in \Theta^\ast }\norm{\theta-\theta^\ast}_2  \leq \left(1+\frac{1}{\delta}\right)\bar{\Delta} \right\}.$$
Finally, we define $d(\theta_{t})=\min_{\theta^\ast \in \Theta^\ast } \norm{{\theta}_t-{\theta}^\ast}_2$. 
We are now ready to present the main result.
\begin{theorem} \label{thm: two phases bound}
Consider an arbitrary tolerance level $\delta>0$ and a small enough tolerance level $\epsilon>0$. For every initial point $\theta_1$, if 
$\theta_{T+1}$ is generated by Algorithm \ref{alg:2 phases} with 
 \begin{align*}
& T_1\geq \left(\frac{6D(\theta_1)}{\delta\epsilon_0}\right)^{\frac{8L}{C^0_\delta \ln2}}, \ T_2 \geq \frac{t_0 \epsilon_0 }{6 \delta\epsilon}-t_0, \  T=T_1+T_2, \\
& B_1\geq \max\left\{ \frac{30\sigma^2}{C^0_\delta \epsilon_0 \delta}, \frac{6\sigma T_1 \log{T_1}}{\bar{\Delta} L}  \right\}, \  B_2 \geq \frac{\sigma^2\ln(T_2+t_0)}{6 C_\alpha \delta  \epsilon}, \\
& \eta_t=\eta\leq \hyuninadd{ \min \left\{\frac{\log{T_1}}{T_1 L},\frac{8}{C^0_\delta}  \right\}}  \text{  for } t\leq T_1, \\
& \eta_t= \frac{1}{t-T_1+t_0} \ \text{  for } t> T_1
 \end{align*}
where
\begin{align} 
&\epsilon_0=\min\left\{ \left(\frac{\lambda \min_s\rho(s)}{6\ln{2}} \right)^2\left(\alpha \exp(\frac{-\bar{r}}{(1-\gamma)\lambda}) \right)^4,1\right\},\\ &t_0\geq \sqrt{\frac{3 \sigma^2}{2\delta \epsilon_0}},\\
&C_\alpha\coloneqq\frac{2\lambda}{|\mathcal{S}|} \min_{s} \rho(s) (1-\alpha)^2\min_{s,a} \pi_{\theta^*}(a|s)^2 \norm{\frac{d_{\rho}^{\pi_\lambda^\ast}}{\rho}}_\infty^{-1}>0,\label{eq: C alpha}\\
& C^0_\delta=\frac{2\lambda}{|\mathcal{S}|} \norm{\frac{d_{\rho}^{\pi_\lambda^\ast}}{\rho}}_\infty^{-1} \min_{s} \rho(s)  \min_{\theta \in \mathcal{G}^0_\delta} \min_{s,a} \pi_\theta(a|s)^2 ,
\end{align}
\myred{and $\sigma$ is defined in Lemma \ref{lemma: bounded variance},
 then we have $\mathbb{P}(D(\theta_{T+1}) \leq \epsilon)\geq 1-\delta.$}
In total, it requires $\Tilde{{\mathcal{O}}}\left( \epsilon^{-2}\right)$ samples to obtain an $\epsilon$-optimal policy with high probability.
\end{theorem}



\begin{algorithm}[ht] 
\caption{Two-phase stochastic PG for entropy regularized RL}
\label{alg:2 phases}
\begin{algorithmic}[1]
\STATE \textbf{Inputs}: $\rho, \lambda$, $\theta_1, B_1, B_2, T_1, T, \{\eta_t\}_{t=1}^T$. 
\FOR{$t = 1, 2, \dots,T$}
\IF{$t\leq T_1$}
\STATE $B=B_1$
\ELSE
\STATE $B=B_2$
\ENDIF
\STATE Run lines 3-7 in Algorithm \ref{alg:unbiased PG}
\ENDFOR{}
\STATE \textbf{Outputs}: $\theta_T$;
\end{algorithmic}
\end{algorithm}



\subsection{Discussion}
{In Theorem \ref{thm: two phases bound}, we have derived strong last-iterate complexity bounds (in contrast to the predominant running-min and ergodic complexity bounds in the reinforcement learning literature), with the desirable $\widetilde{\mathcal{O}} (1/\epsilon^2)$ dependency on the targeting tolerance $\epsilon$. 
That being said, the polynomial dependency on $1/\delta$ and exponential dependencies on other problem- and algorithm-dependent constants also indicate that our bounds may not be tight in general. 

The convergence analysis of the stochastic softmax PG with the entropy regularization is challenging \footnote{\myred{Note that similar difficulties in generalization from exact policy gradients to stochastic policy gradients have been observed in \cite{mei2021understanding}, which states that ``unlike the \myred{exact} gradient setting, geometric information cannot be easily exploited in the stochastic case for accelerating policy optimization without detrimental consequences or impractical assumptions''.}}
due to the weaker regularization effect of the entropy regularization (compared to the log-barrier regularization adopted in previous works on global optimality convergence of policy gradient methods \cite{agarwal2020optimality,zhang2020sample}), as well as the ``softmax
gravity well'' induced by the softmax parameterization which has also been observed in the exact gradient setting \cite{mei2020global, mei2020escaping}. 
In particular, it only entails uniform gradient domination properties for policies that are bounded below uniformly (\textit{cf.} Lemma \ref{lemma: non-uniform lojasiewicz}). We thus need to control the trajectory to ensure that $\pi_{\theta_t}$ remains in the region where it is uniformly bounded from below for all $t$. However, even with large batches, it is generally difficult to control stochastic trajectories, which eventually leads to the polynomial dependency on $1/\delta$ and the exponential dependencies on some constants. If large batches are not used, then the trajectories would be even harder to control and no guarantees may be attained  unless 
additional structural assumptions are enforced on the underlying MDP. \mynewred{In addition, because of the different batch sizes and analysis techniques used in two  phases, the conditions for the step-size $\eta_t$ are also specific for the corresponding phase.}

In the next three sections, we provide the proof of Theorem \ref{thm: two phases bound}. We begin by showing that the iterates will converge to a neighborhood of the optimal solution with high probability in Section \ref{sec: global convergence with arb init}, and then utilize the curvature information around the optimal policy to guarantee that the action probabilities will still remain uniformly bounded with
high probability in Section \ref{Uniformly bounded action probabilities given a good initialization}. We then combine the two steps to prove Theorem \ref{thm: two phases bound} in Section \ref{sec: proof of thm 1}. \myred{For a roadmap of the main ideas behind the proof and the utilization of different lemmas in the paper, please refer to Figure \ref{fig: Roadmap}.}
\begin{figure}[ht]
\centering
\includegraphics[width=8.8cm]{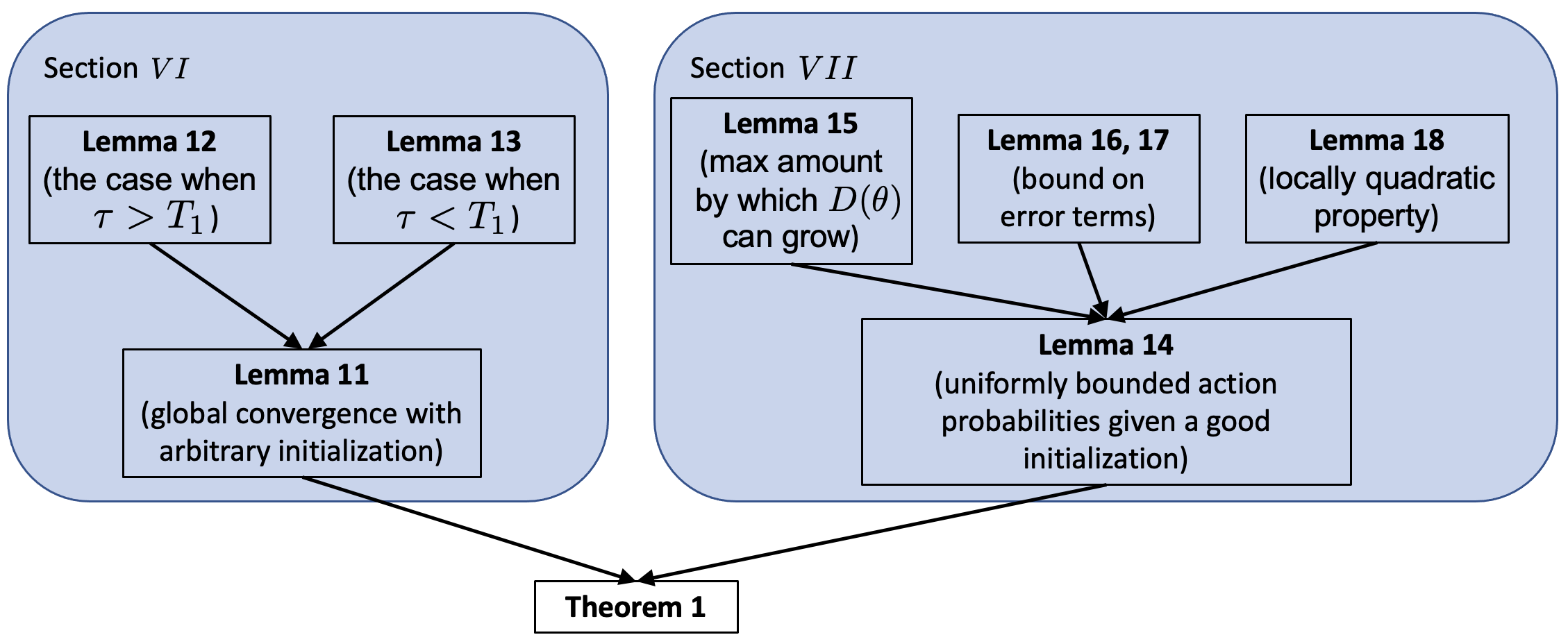}
\caption{\myred{Roadmap of the main ideas behind the proof of Theorem \ref{thm: two phases bound} and their connection to various lemmas in the paper.}}
\label{fig: Roadmap}
\end{figure}
}

\section{Experiment}\label{sec:experiment}

\begin{figure}
    \centering
    \includegraphics[width=8.8cm]{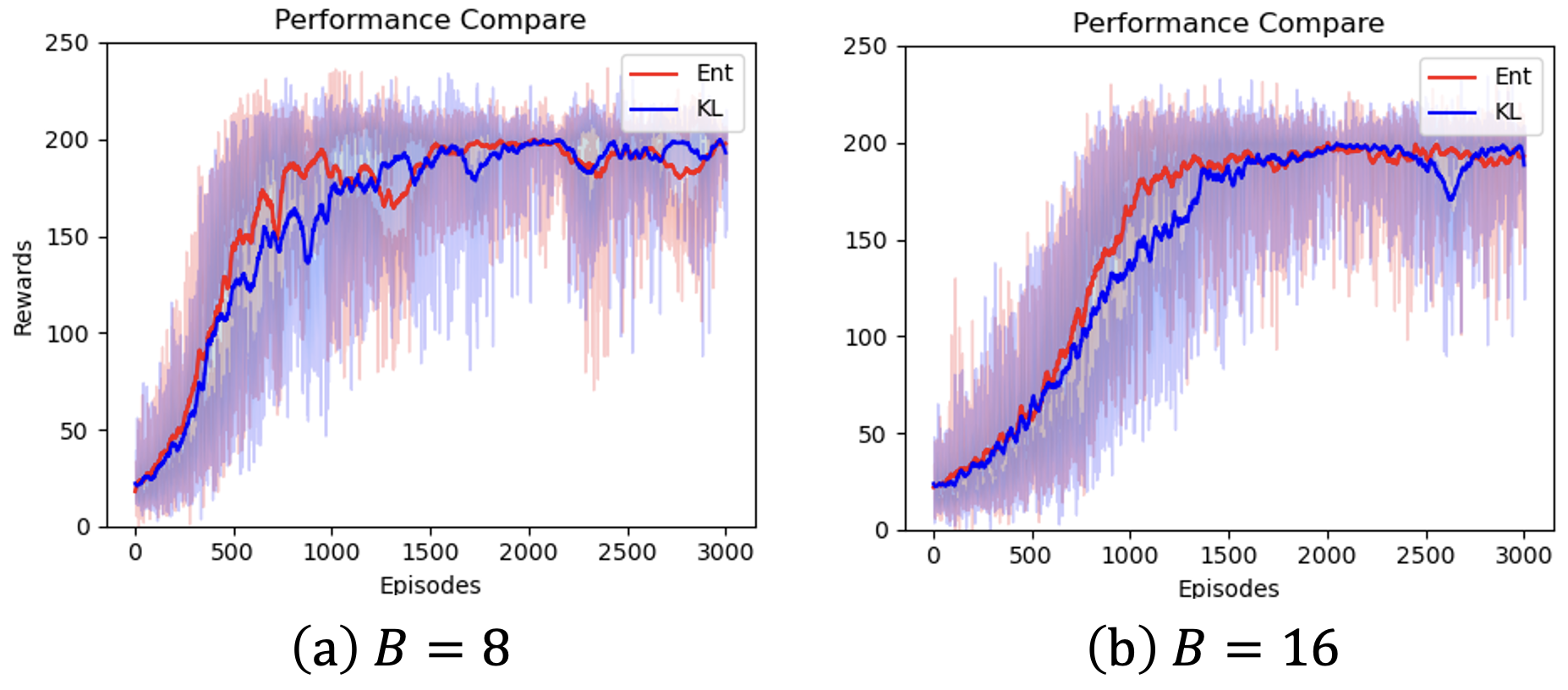}
    \caption{\hyuninadd{Rewards comparison among two different value estimators.} }
    \label{fig:compare_estimator}
\end{figure}

We compare our proposed PG estimator with the estimator given in \cite{schulman2017equivalence} for PG regularized by KL-divergence between the current policy and the reference policy.
The two state-value estimators are evaluated in a cartpole environment of the Mujoco package. Experiments are performed to compare the performance of the estimators for two different batch sizes $ B \in \{ 8,16 \}$ when $\lambda=0.1$ (see Fig. \ref{fig:compare_estimator}). For each batch size, experiments are repeated with five different seeds. For each subfigure, the solid lines are the means of the experiments among five different seeds, and the shaded area is a confidence interval within one standard deviation. The red line is our proposed method, and the blue line is the KL-divergence-based estimator given in \hl{the entropy-regularized reinforcement learning algorithm} \cite{schulman2017equivalence}. Fig. \ref{fig:compare_estimator} shows that our proposed estimator converges to the rewards faster than the estimator of \hl{the entropy-regularized reinforcement learning algorithm} \cite{schulman2017equivalence}, which supports the results of this paper that the proposed state value estimator can better evaluate the policy than the KL-divergence-based estimator.
\section{Global convergence with arbitrary initialization} \label{sec: global convergence with arb init}
In this section, we provide the first step towards the  proof of Theorem \ref{thm: two phases bound}. In particular, we will prove that after the first phase of Algorithm \ref{alg:2 phases}, the iterates will converge to a neighborhood of the optimal solution with high probability due to the use of a large batch size. 

With a large batch size, we can show that if the iterations with
the exact PG are bounded, then the iterations with the unbiased stochastic PG will
remain bounded with high probability. This will further imply that the unbiased stochastic PG will converge to the neighborhood of the globally optimal policy with high
probability. This is a non-trivial result involving the stopping/hitting time analysis, as presented below. 

\begin{lemma} \label{thm: high complexity global convergence}
Consider arbitrary tolerance levels $\delta>0$ and $\epsilon_0>0$. For every initial point $\theta_1$, if 
 $\theta_T$ is generated by Algorithm \ref{alg:unbiased PG} with $\eta_t=\eta \leq \hyuninadd{\min \left\{\frac{\log{T_1}}{T_1 L},\frac{8}{C^0_\delta}  \right\}}$, 
$T_1=\left(\frac{6D(\theta_1)}{\delta\epsilon_0}\right)^{\frac{8L}{C^0_\delta \ln2}}$, and $B_1= \max\left\{ \frac{30\sigma^2}{C^0_\delta \epsilon_0 \delta}, \frac{6\sigma}{\bar{\Delta} L} \cdot T_1\cdot \log{T_1} \right\}$,
 then we have $\mathbb{P}(D(\theta_{T_1}) \leq \epsilon_0)\geq 1-\delta/2.$
\end{lemma}

\subsection{Helpful lemmas}

 \myred{To prove Lemma \ref{thm: high complexity global convergence}, we consider the case when $\tau>T_1$ and the case when $\tau\leq T_1$ separately, where $\tau$ is defined in \eqref{eq: stopping time}.
When $\tau>T_1$,  we can use Lemma \ref{lemma: non-uniform lojasiewicz} to show that  $D(\theta_{t})$ is linearly convergent up to some aggregated estimation error.} 

\begin{lemma} \label{lemma: delta under tau >T}
If $\eta_t{=\eta}\leq \hyuninadd{\min \left\{  \frac{1}{2L} ,\frac{8}{C^0_\delta}  \right\}}$,  then
$
\EE[D(\theta_{T_1}) \mathbf{1}_{\tau>T_1}]\leq  \left(1-\frac{\eta C^0_\delta}{8}\right)^{T_1-1} D(\theta_1)+ \frac{5 \sigma^2}{8C^0_\delta B_1}.
$
\end{lemma}
\begin{proof}
Let $e_t=\nabla V_\lambda^{\theta_t}(\rho)-u_t$, where $u_t=\frac{1}{B_1} \sum_{i=1}^{B_1}\hat{\nabla} V_\lambda^{\theta_t,i}(\rho)$ and $\hat{\nabla} V_\lambda^{\theta_t,i}(\rho)$ is an unbiased estimator of $\nabla V_\lambda^{\theta_t}(\rho)$. 
Since ${\nabla} V^{\theta}_\lambda(\rho)$ is $L$-smooth due to Lemma \ref{lemma: Lipschitz-Continuity of Policy Gradient}, it follows from Lemma \ref{lemma: due to smoothness} in the supplementary material:
\begin{align*}
 &\EE^t[   D(\theta_{t+1})-D(\theta_{t})]\mathbf{1}_{\tau>t}\\
 =&  \EE^t\left[V_\lambda^{\theta_t}(\rho)-V_\lambda^{\theta_{t+1}}(\rho)\right]\mathbf{1}_{\tau>t}\\
    \leq &  \EE^t\left[-\frac{\eta}{8} \norm{u_t}_2^2+\frac{3\eta}{4}\norm{e_t}_2^2\right]\mathbf{1}_{\tau>t}\\
    \leq &  \EE^t\left[-\frac{\eta}{8} \norm{u_t-\nabla V_\lambda^{\theta_t}(\rho)+ \nabla V_\lambda^{\theta_t}(\rho)}_2^2+\frac{3\eta}{4}\norm{e_t}_2^2\right]\mathbf{1}_{\tau>t}\\
    = & \EE^t\left[-\frac{\eta}{8} \norm{ \nabla V_\lambda^{\theta_t}(\rho)}_2^2+\frac{5\eta}{8}\norm{e_t}_2^2\right]\mathbf{1}_{\tau>t}\\
    \leq & \EE^t\left[-\frac{\eta C(\theta_t)}{8} D(\theta_{t})+\frac{5\eta}{8}\norm{e_t}_2^2\right]\mathbf{1}_{\tau>t},
\end{align*}
for every $\eta\leq \frac{1}{2L}$, 
where the second inequality uses the fact that $u_t$ is an unbiased estimator of $\nabla V_\lambda^{\theta_t}(\rho)$ and the last inequality is due to Lemma \ref{lemma: non-uniform lojasiewicz}.
We now consider two cases:
\begin{itemize}
    \item Case 1: Assume that $\tau>t$, which implies that $\theta_t \in \mathcal{G}^0_\delta$ and $C(\theta_t)\geq C^0_\delta$. Then, we have
$\EE[D(\theta_{t+1})|\mathcal{F}_t] \leq \left(1-\frac{\eta C^0_\delta}{8}\right) D(\theta_{t})+\frac{5\eta}{8}\EE\left[\norm{e_t}_2^2|\mathcal{F}_t\right].$
    \item Case 2: Assume that $\tau\leq t$ which leads to
$\EE[D(\theta_{t+1})|\mathcal{F}_t]\mathbf{1}_{\tau>t}=0$.
\end{itemize}
Now combining the above two cases yields the inequality
\begin{align*}
    &\EE[D(\theta_{t+1})|\mathcal{F}_t]\mathbf{1}_{\tau>t}\\
    \leq& \left\{(1-\frac{\eta C^0_\delta}{8}) D(\theta_{t})+\frac{5\eta}{8}\EE\left[\norm{e_t}_2^2|\mathcal{F}_t\right] \right\}\mathbf{1}_{\tau>t}\\
    \leq & \left(1-\frac{\eta C^0_\delta}{8}\right) D(\theta_{t}) \mathbf{1}_{\tau>t}+\frac{5\eta}{8}\EE\left[\norm{e_t}_2^2|\mathcal{F}_t\right].
\end{align*}

In addition, conditioning on $\mathcal{F}_t$ yields that
\begin{align*}
    \EE[D(\theta_{t+1})\mathbf{1}_{\tau>t+1} |\mathcal{F}_t]& \leq  \EE[D(\theta_{t+1})\mathbf{1}_{\tau>t}|\mathcal{F}_t]\\
    &=\EE[D(\theta_{t+1})|\mathcal{F}_t]\mathbf{1}_{\tau>t},
\end{align*}
where the last equality uses the fact that $\tau$ is a stopping time and the random variable $\mathbf{1}_{\tau>t}$ is determined completely by the sigma-field $\mathcal{F}_t$.
Taking the expectations over the sigma-field $\mathcal{F}_t$ and then arguing inductively gives rise to
\begin{align*}
    &\EE[D(\theta_{t+1}) \mathbf{1}_{\tau>t+1}]\\
    \leq & \prod_{i=0}^t \left(1-\frac{\eta C^0_\delta}{8}\right) D(\theta_1) +\sum_{i=0}^t\left(1-\frac{\eta C^0_\delta}{8}\right)^i\frac{5\eta}{8}\EE\left[\norm{e_i}_2^2\right] \\
    \leq & \left(1-\frac{\eta C^0_\delta}{8} \right)^{t} D(\theta_1) +\frac{5 \sigma^2}{C^0_\delta B_1}.
\end{align*}
By setting $t+1=T_1$, we obtain that
$
\EE[D(\theta_{T_1}) \mathbf{1}_{\tau>T_1}]\leq  \left(1-\frac{\eta C^0_\delta}{8}\right)^{T_1-1} D(\theta_1) +\frac{5\sigma^2 }{C^0_\delta B_1}.
$
This completes the proof.
\end{proof}

We now establish that $\{\theta_t\}_{t=1}^T$ will be bounded with high probability if the large batch size is used.


\begin{lemma}\label{lemma: prob tau <T}
It holds that 
$
\mathbb{P}(\tau \leq T_1 )\leq  \frac{ \delta \cdot\eta\cdot T_1\cdot (1+\eta L)^{T_1-1}\cdot{\sigma}  }{\bar{\Delta} B_1}.
$
\end{lemma}
\begin{proof}
By the triangle inequality and the fact that the iterations of the algorithm with the exact PG are bounded by  $\bar{\Delta}$, we have
\begin{align*}
d(\theta_t) \leq \norm{\theta_t -\bar{\theta}_t}_2+\min_{\theta^\ast \in \Theta^\ast }\norm{\theta^\ast -\bar{\theta}_t}_2\leq \norm{\theta_t -\bar{\theta}_t}_2+\bar{\Delta}.
\end{align*}

Using the update rule of the algorithm with the exact PG $\nabla V_\lambda^{\bar{\theta}_{i}}(\rho)$ and the stochastic PG $u_i=\frac{1}{B_1} \sum_{j=1}^{B_1}\hat{\nabla} V_\lambda^{\theta_i,j}(\rho)$, one can write
\begin{align*}
d(\theta_i) = & \norm{(\theta_1 +\sum_{i=1}^{t-1}\eta_i  u_i ) -(\theta_1+ \sum_{i=1}^{t-1}\eta_i \nabla V_\lambda^{\bar{\theta}_{i}}(\rho) )}_2+\bar{\Delta}\\
\leq & \sum_{i=1}^{t-1}\eta_i \norm{ u_i  - \nabla V_\lambda^{\bar{\theta}_{i}}(\rho)}_2+\bar{\Delta}\\
= & \sum_{i=1}^{t-1}\eta_i \norm{ u_i  - \nabla V_\lambda^{{\theta}_{i}}(\rho)+\nabla V_\lambda^{{\theta}_{i}}(\rho)- \nabla V_\lambda^{\bar{\theta}_{i}}(\rho)}_2+\bar{\Delta}\\
\leq & \sum_{i=1}^{t-1}\eta_i \norm{e_i}_2+\sum_{i=1}^{t-1}\eta_i L\norm{\theta_i-\bar{\theta}_i}_2+\bar{\Delta}.
\end{align*}
By expanding $\norm{\theta_i-\bar{\theta}_i}_2$ recursively, it can be concluded that
\begin{align*}
&d(\theta_t)\\
\leq & \sum_{i=1}^{t-1}\eta_i \norm{e_i}_2+\eta_{t-1} L\norm{\theta_{t-1}-\bar{\theta}_{t-1}}_2+\sum_{i=1}^{t-2}\eta_i L\norm{\theta_i-\bar{\theta}_i}_2+\bar{\Delta}\\
\leq & \sum_{i=1}^{t-1}\eta_i \norm{e_i}_2+\eta_{t-1} L\sum_{i=1}^{t-2}\eta_i \norm{e_i}_2+\eta_{t-1} L^2\sum_{i=1}^{t-2} \eta_i\norm{\theta_i-\bar{\theta}_i}_2\\
&+\sum_{i=1}^{t-2}\eta_i L\norm{\theta_i-\bar{\theta}_i}_2+\bar{\Delta}\\
= & \sum_{i=1}^{t-1} \eta_i\norm{e_i}_2+\eta_{t-1} L\sum_{i=1}^{t-2} \eta_i \norm{e_i}_2\\
&+\sum_{i=1}^{t-2}\left(\eta_i L+\eta_{t-1}\eta_i L^2\right)\norm{\theta_i-\bar{\theta}_i}_2+\bar{\Delta}\\
\leq & \sum_{i=1}^{t-1}\eta_i \norm{e_i}_2+\eta_{t-1} L\sum_{i=1}^{t-2}\eta_i \norm{e_i}_2+\bar{\Delta}\\
&+\left(\eta_{t-2} L+\eta_{t-1}\eta_{t-2} L^2\right)\sum_{i=1}^{t-3}\eta_i\norm{e_i}_2\\
&+ \sum_{i=1}^{t-3} \left(\left(\eta_{t-2} L+\eta_{t-1}\eta_{t-2} L^2\right)\eta_iL \right.\\
&\left.\hspace{4cm}+\left(\eta_{i} L+\eta_{t-1}\eta_{i} L^2\right)\right) \norm{\theta_i-\bar{\theta}_i}_2\\
\leq &\bar{\Delta}+ \sum_{i=1}^{t-1}\eta_i \norm{e_i}_2+\eta_{t-1} L\sum_{i=1}^{t-2}\eta_i \norm{e_i}_2\\
&+\left(\eta_{t-2} L+\eta_{t-1}\eta_{t-2} L^2\right)\sum_{i=1}^{t-3}\eta_i\norm{e_i}_2\\
&+\sum_{i=1}^{t-4} \left(\left(\eta_{t-2} L+\eta_{t-1}\eta_{t-2} L^2\right)\eta_{t-3}L\right.\\
&\left.\hspace{3cm}+\left(\eta_{t-3} L+\eta_{t-1}\eta_{t-3} L^2\right)\right) \norm{e_i}_2\\
&+ \sum_{i=1}^{t-4} \left(\left(\eta_{t-2} \eta_{t-3}L^2+\eta_{t-1}\eta_{t-2}\eta_{t-3} L^3\right)\eta_i L \right.\\
&\hspace{3cm}+\left.\left(\eta_{t-3} L+\eta_{t-1}\eta_{t-3} L^2\right)\eta_i\right) \norm{\theta_i-\bar{\theta}_i}_2\\
= & \sum_{i=1}^{t-1}  \eta_i \prod_{j=i+1}^{t-1}(1+\eta_j L) \norm{e_i}_2+\bar{\Delta}.
\end{align*}
Then, by the definition of $\tau$ in \eqref{eq: stopping time} and Markov inequality, we obtain
\begin{align*}
\mathbb{P}(\tau \leq T_1 )=& \mathbb{P}(\max_{t\in\{1,\ldots,T_1\}}d(\theta_t) \geq (1+\frac{1}{\delta})\bar{\Delta} )\\
\leq &\mathbb{P}(\sum_{i=1}^{T_1-1}  \eta_i \prod_{j=i+1}^{T_1-1}(1+\eta_j L) \norm{e_i}_2+\bar{\Delta} \geq (1+\frac{1}{\delta})\bar{\Delta})\\
\leq & \frac{ \sum_{i=1}^{T_1-1}  \eta_i \prod_{j=i+1}^{T_1-1}(1+\eta_j L) \EE[\norm{e_i}_2] }{\frac{1}{\delta}\bar{\Delta}}\\
\leq & \frac{ \delta\eta (1+\eta L)^{T_1-1} \sum_{i=1}^{T_1-1} \EE[\norm{e_i}_2]}{\bar{\Delta}},
\end{align*}
where we use the fact that $\eta_t=\eta$ for all $t\in\{1,2,\ldots\}$. Furthermore, since $\EE[\norm{e_i}_2]\leq \sqrt{\EE[\norm{e_i}^2_2]}\leq \frac{\sigma}{B_1}$, we have
$\mathbb{P}(\tau \leq T_1 )
\leq  \frac{ \delta \cdot\eta\cdot T_1\cdot (1+\eta L)^{T_1-1}\cdot{\sigma}  }{\bar{\Delta} B_1}.$
This completes the proof.
\end{proof}


\subsection{Proof of Lemma \ref{thm: high complexity global convergence}}
By combining Lemmas \ref{lemma: delta under tau >T} and \ref{lemma: prob tau <T}, we obtain that
\begin{align*}
& \mathbb{P}(D(\theta_{T_1})\geq \epsilon_0) \\
\leq  &   \mathbb{P}({\tau>T_1}, D(\theta_{T_1}) \geq \epsilon_0)+\mathbb{P}( {\tau\leq T_1}, D(\theta_{T_1}) \geq \epsilon_0)\\
\leq & \frac{\EE[\mathbf{1}_{\tau>T_1} D(\theta_{T_1})]}{\epsilon_0}+\mathbb{P}({\tau\leq T_1})\\
\leq &\left(1-\frac{\eta C^0_\delta }{8}\right)^{T_1-1} \frac{D(\theta_1)}{\epsilon_0} +\frac{5\sigma^2}{C^0_\delta B_1\epsilon_0}\\
&+ \frac{\delta \cdot \eta\cdot T_1\cdot (1+\eta L)^{T_1-1}\cdot\sigma  }{\bar{\Delta}B_1}\\
\leq &\left(1-\frac{\eta C^0_\delta }{8}\right)^{\frac{8}{\eta C^0_\delta } \frac{\eta C^0_\delta  T_1}{8}} \frac{D(\theta_1)}{\epsilon_0} \\
&+\frac{5\sigma^2}{C^0_\delta B_1\epsilon_0}+ \frac{ \delta \cdot \eta\cdot T_1\cdot (1+\eta L)^{T_1-1}\cdot\sigma }{\bar{\Delta} B_1}\\
\leq &\frac{1}{2}^{\frac{\eta C^0_\delta  T_1}{8}} \frac{D(\theta_1)}{\epsilon_0} +\frac{5\sigma^2}{C^0_\delta B_1\epsilon_0}+ \frac{ \delta \cdot \eta\cdot T_1\cdot (1+\eta L)^{T_1-1}\cdot\sigma }{\bar{\Delta} B_1},
\end{align*}
where the second inequality holds due to the Markov inequality, and
the last inequality holds because of $(1-\frac{1}{m})^m \leq \frac{1}{2}$ for all $m\geq 1$ and $\frac{8}{\eta C^0_\delta } \geq 1$. \hyuninadd{For any $x \in \mathbb{R}$ that satisfies $x>0$, the inequality $(\log{x})/x - 1/2 <0$ \first{also} holds. Therefore, for any $T_1 >0$, the inequality $\frac{\log{T_1}}{T_1 L} - \frac{1}{2L} <0$ always holds}. By taking $\eta\leq \hyuninadd{\min \left\{\frac{\log{T_1}}{T_1 L},\frac{8}{C^0_\delta}  \right\}}$, we obtain
\begin{align*}
&\mathbb{P}(D(\theta_{T_1}) \geq \epsilon_0)\\
\leq &\frac{1}{2}^{\frac{C^0_\delta  \log{T_1}}{8L}} \frac{D(\theta_1)}{\epsilon_0} +\frac{5\sigma^2}{C^0_\delta B_1\epsilon_0}+ \frac{ \delta \cdot\log{T_1}\cdot  (1+\frac{\log{T_1}}{T_1})^{T_1-1}\cdot{\sigma}  }{\bar{\Delta} B_1L}\\
\leq &\frac{1}{2}^{\frac{C^0_\delta  \log{T_1}}{8L}} \frac{D(\theta_1)}{\epsilon_0} +\frac{5\sigma^2}{C^0_\delta B_1\epsilon_0} \\
&+ \frac{ \delta \cdot\log{T_1}\cdot (1+\frac{\log{T_1}}{T_1})^{\frac{T_1}{\log{T_1}}\cdot \log{T_1}}\cdot{\sigma}  }{\bar{\Delta} B_1L}\\
\leq &\frac{1}{2}^{\frac{C^0_\delta  \log{T_1}}{8L}} \frac{D(\theta_1)}{\epsilon_0} +\frac{5\sigma^2}{C^0_\delta B_1\epsilon_0}+ \frac{\delta \cdot \log{T_1}\cdot T_1\cdot{\sigma}  }{\bar{\Delta} B_1L}\\
\leq &\frac{1}{T_1^{\frac{\ln{2} C^0_\delta }{8L}}} \frac{D(\theta_1)}{\epsilon_0} +\frac{5\sigma^2}{C^0_\delta B_1\epsilon_0}+ \frac{\delta \cdot \log{T_1}\cdot T_1\cdot{\sigma}  }{\bar{\Delta} B_1L},
\end{align*}
where we have used $(1+x)^{1/x}\leq e$ in the third inequality and $a^{\ln{b}}=b^{\ln{a}}$ in the last inequality. To guarantee $\mathbb{P}(D(\theta_{T_1}) \geq \epsilon_0)\leq \delta/2$, it suffices to have
\begin{align*}
    T_1=\left(\frac{6D(\theta_1)}{\delta\epsilon_0}\right)^{\frac{8L}{C^0_\delta \ln2}}, B_1= \max\left\{ \frac{30\sigma^2}{C^0_\delta \epsilon_0 \delta}, \frac{6\sigma}{\bar{\Delta} L} \cdot T_1\cdot \log{T_1} \right\}.
\end{align*}
This completes the proof.

\section{Uniformly bounded action probabilities given a good initialization}
\label{Uniformly bounded action probabilities given a good initialization} In this section, we will show how to utilize the curvature information around the optimal policy to guarantee that the action probabilities will still remain uniformly bounded with high probability, which serves as the second step towards the proof of Theorem \ref{thm: two phases bound}.

\begin{lemma} \label{lemma: pi lower bound high prob}
Given a tolerance level $\delta>0$, let $\pi_{\lambda}^*$ be the optimal policy of $V_\lambda^\theta(\rho)$. \mynewred{Assume further that the random variable $\{\theta_t\}_{t=1}^{T_2}$ is generated from Algorithm \ref{alg:unbiased PG} with a step-size sequence of the form $\eta_{t}=1 /(t+t_0)$ and a batch-size sequence $B\geq \frac{1}{\eta_t}$ for all $t=1,2,\ldots,T_2$. If $t_0\geq \sqrt{\frac{3 \sigma^2}{2\delta \epsilon_0}}$, 
and  $\pi_{\theta_{1}}$ is initialized in a neighborhood $\mathcal{U}_{1} $ such that
\begin{align}\label{eq: U1}
\mathcal{U}_{1} = \left\{\pi \in \Delta(\mathcal{A})^{|\mathcal{S}|}:D(\pi)\leq \epsilon_0\right\},
\end{align}
where $\epsilon_0=\min\left\{ \left(\frac{\lambda \min_s\rho(s)}{6\ln{2}} \right)^2\left(\alpha \exp(\frac{-\bar{r}}{(1-\gamma)\lambda}) \right)^4,1\right\}$ and the constant $\alpha\in (0,1)$, then the event \footnote{\myred{The $\sigma$-}\hyuninadd{field} \myred{of this event is the Cartesian product of the natural Borel $\sigma$-}\hyuninadd{field} \myred{on the
underlying MDPs \cite[Section 2.1.6]{puterman2014markov}}.}
\begin{align}\label{eq: definition of Omega alpha}
\Omega_{\alpha,1}^{T_2}=\left\{\min_{s,a} \pi_{\theta_t}(a|s)\geq (1-\alpha) \min_{s,a} \pi_{\lambda}^*(a|s), \forall t=1,2, \ldots,T_2 \right.\Big\}
\end{align}
occurs with probability at least $1-\delta/6$.}
\end{lemma}

\subsection{Helpful lemmas}
To prove Lemma \ref{lemma: pi lower bound high prob},
we first characterize the maximum amount  by which $D(\theta_{t})$ can grow at each step.  
\begin{lemma}\label{lemma: D t+1 and D t}
Suppose that $\{\theta_t\}$ is generated by Algorithm \ref{alg:unbiased PG} with $0<\eta_t\leq\frac{(1-\gamma)^3}{16\bar{r}+\lambda(8+16\log |\mathcal{A}|) }$ for all $t\geq 1$. We have
\begin{align} \label{eq: D t+1 and D t}
 D(\theta_{t+1})
    \leq & \left(1-\frac{\eta_t C(\theta_t)}{4} \right)D(\theta_{t})-
    \frac{\eta_t}{2}\xi_t+\frac{\eta_t}{4}\norm{e_t}_2^2,
\end{align}
where $\xi_t=\inner{e_t}{{\nabla}V^{\theta_t}_\lambda(\rho)}$ and $e_t=\hat{\nabla}V_\lambda^{\theta_t}(\rho)-{\nabla}V_\lambda^{\theta_t}(\rho)$.
\end{lemma}
\begin{proof}
Since ${\nabla} V^{\theta}_\lambda(\rho)$ is $L$-smooth in light of  Lemma \ref{lemma: Lipschitz-Continuity of Policy Gradient}, it follows  from Lemma \ref{lemma: due to smoothness} that
\begin{align*}
 &D(\theta_{t+1})- D(\theta_{t})\\
 \leq &  -\frac{\eta_t}{4} \norm{\hat{\nabla} V^{\theta}_\lambda(\rho)}_2^2+\frac{\eta_t}{2}\norm{e_t}_2^2\\
    \leq &  -\frac{\eta_t}{4} \norm{\hat{\nabla} V^{\theta_t}_\lambda(\rho)-{\nabla} V^{\theta_t}_\lambda(\rho)+ {\nabla} V^{\theta_t}_\lambda(\rho)}_2^2
    +\frac{\eta_t}{2}\norm{e_t}_2^2\\
    = &  -\frac{\eta_t}{4} \norm{\hat{\nabla} V^{\theta_t}_\lambda(\rho)-{\nabla} V^{\theta_t}_\lambda(\rho)}_2^2-\frac{\eta_t}{4} \norm{ \hat{\nabla} V^{\theta_t}_\lambda(\rho)}_2^2\\
    &- \frac{\eta_t}{2}\inner{e_t}{{\nabla} V^{\theta_t}_\lambda(\rho)}+\frac{\eta_t}{2}\norm{e_t}_2^2\\
    = & -\frac{\eta_t}{4} \norm{ {\nabla} V^{\theta_t}_\lambda(\rho)}_2^2- 
    \frac{\eta_t}{2}\inner{e_t}{{\nabla} V^{\theta_t}_\lambda(\rho)}+\frac{\eta_t}{4}\norm{e_t}_2^2\\
    \leq & -\frac{\eta_t C(\theta_t)}{4}  D(\theta_{t})-
    \frac{\eta_t}{2}\inner{e_t}{{\nabla} V^{\theta_t}_\lambda(\rho)}+\frac{\eta_t}{4}\norm{e_t}_2^2,
\end{align*}
for every $\eta_t\leq \frac{1}{2L }$,
where  the last inequality is due to Lemma \ref{lemma: non-uniform lojasiewicz}. 
\end{proof}

The quantity by which $D(\theta_{t})$ can grow at each step can be large for any given $t$ but we will show that, with high probability, the aggregation of these errors remains controllably small under the stated conditions on the step-sizes and batch size. 

Similar as the techniques used in \cite{hsieh2019convergence, hsieh2020explore, mei2021leveraging, mertikopoulos2019learning, mertikopoulos2020almost},
we now encode the error terms in \eqref{eq: D t+1 and D t} as
$M_n=\sum_{t=1}^n \eta_t    \xi_t$
and $S_n=\sum_{t=1}^n \frac{\eta_t}{4}\norm{e_t}_2^2.$ 

\hyuninadd{For given $\pi_{\theta^n}$, the parameter $\xi_n$ is fully determined by a trajectory obtained at iteration $n$, i.e., $\tau_n = \{ s^n_0\,a^n_0,...,s^n_H\}$. For a stationary environment, the trajectory $\tau_n$ is fully determined by the policy at time $n$, namely $\pi^n$, which is parameterized by $\theta_n$ where $\theta_n$ is fully determined by the update rule by algorithm \ref{alg:unbiased PG} (\textbf{Ent-RPG}) . Note that the input of the update rule is fully determined by the information up to the $n-1$. Then, $\theta_n$ is measurable with respect to $\mathcal{F}_{n-1}$, which subsequently concludes that $\{\xi_1,\xi_2,...,\xi_n \}$ are also measurable regarding with $\mathcal{F}_{n-1}$. Therefore $\mathbb{E}^{n-1}\left[\xi_{n}\right]=0$ since  
    \begin{align*}
        \mathbb{E}^{n-1}\left[\xi_{n}\right] =& \mathbb{E} \left[ \left\langle \hat{\nabla} V^{\theta^n}_\lambda (\rho) - \nabla V^{\theta^n}_\lambda (\rho),~\nabla V^{\theta^n}_\lambda (\rho) \right\rangle \Big| \mathcal{F}_{n-1} \right]  \\
         =& \mathbb{E} \left[ \left\langle \nabla V^{\theta^n}_\lambda (\rho)  ,\nabla V^{\theta^n}_\lambda (\rho)  \right\rangle \Big| \mathcal{F}_{n-1} \right]  \\
         & - \mathbb{E} \left[ \left| \left| \nabla V^{\theta^n}_\lambda (\rho) \right | \right | ^2 \Big| \mathcal{F}_{n-1} \right]  \\
         =&  0
    \end{align*}
holds. Then, we have $\mathbb{E}^{n-1}\left[M_{n} \right]=M_{n-1}$. Therefore, $M_{n}$ is a zero-mean martingale; likewise, $\mathbb{E}^{n-1}\left[S_{n} \right] \geq S_{n-1}$, and therefore, $S_{n}$ is a submartingale.} \myred{The difficulty of controlling the errors in $M_n$ and $S_n$ lies in the fact that the estimation error $e_n$ may be unbounded. Because of this, we need to take a less direct, step-by-step approach to bound the total error increments conditioned on the event that $D(\theta_n)$ remains close to $D(\theta^\ast)$.
We begin by introducing the ``cumulative mean square error''
$R_{n}=M_n^2+S_n$. \hl{In Lemma \ref{lemma: D t+1 and D t}, we begin with the decomposition of $R_n$ into the terms $M_{n-1}$, $\xi_n$, and $e_n$, and then demonstrate that $R_n$ can be bounded as outlined in Lemma \ref{lemma: three claims}. It is worth noting that constructing new properties of $R_n$ through Lemma \ref{lemma: D t+1 and D t} is necessary — not just to demonstrate the submartingale property of $R_n$ but also to prove Lemma \ref{lemma: pi lower bound high prob} and subsequently Theorem \ref{thm: two phases bound}. Specifically, merely employing the submartingale property of $R_n$ proves insufficient for establishing Lemma \ref{lemma: En high prob}, Lemma \ref{lemma: pi lower bound high prob}, and Theorem \ref{thm: two phases bound}. This necessitates adopting a \emph{less direct} approach that entails demonstrating more sophisticated properties which extend beyond the submartingale characteristic. Consequently, a \emph{step-by-step approach} is necessitated, involving the construction of a new property of $R_n$ that demonstrates high concentration as delineated in Equation \eqref{eq: third argument}. This property is then utilized to substantiate Lemma \ref{lemma: En high prob}, Lemma \ref{lemma: pi lower bound high prob}, and ultimately Theorem \ref{thm: two phases bound}.}
By construction, we have
\begin{align*}
R_{n} &=\left(M_{n-1}+\eta_{n} \xi_{n}\right)^{2}+S_{n-1}+\frac{1}{4} \eta_{n}\left\|e_{n}\right\|^{2} \\
&=R_{n-1}+2 M_{n-1} \eta_{n} \xi_{n}+\eta_{n}^{2} \xi_{n}^{2}+\frac{1}{4} \eta_{n}\left\|e_{n}\right\|^{2}.
\end{align*}
Hence,
$\mathbb{E}\myred{^{n-1}}\left[R_{n} \right]=R_{n-1}+2 M_{n-1} \eta_{n}$} $\myred{\mathbb{E}}^{\hyuninadd{n-1}}$ \myred{$\left[\xi_{n} \right]+\eta_{n}^{2}$}  $\myred{\mathbb{E}}^{\hyuninadd{n-1}}$ \myred{$\left[\xi_{n}^{2}\right]+\frac{1}{4}\eta_{n} $} $\myred{\mathbb{E}}^{\hyuninadd{n-1}}$ \myred{$\left[\left\|e_{n}\right\|^{2}\right] \geq R_{n-1},$
i.e., $R_{n}$ is a submartingale.}
With a fair degree of hindsight, we define $\mathcal{U}$ as:
\begin{align}\label{eq: U}
 \mathcal{U}=\left\{\pi \in \Delta(\mathcal{A})^{|\mathcal{S}|}: D(\pi)\leq 2\epsilon_0+\sqrt{\epsilon_0}\right\}.
\end{align}

To condition it further, we also define the events
\begin{align*}
&\Omega_{n} \equiv \Omega_{n}(\epsilon_0)=\left\{ \pi_{\theta_t}\in\mathcal{U} \text{ for all } t=1,2, \ldots, n\right\}\\
& E_{n} \equiv E_{n}(\epsilon_0)=\left\{R_{t} \leq \epsilon_0 \text { for all } t=1,2, \ldots, n\right\}
\end{align*}
By definition, we also have $\Omega_{0}=E_{0}=\Omega$ (because the set-building index set for $k$ is empty in this case, and every statement is true for the elements of the empty set). These events will play a crucial role in the sequel as indicators of whether $\pi_{\theta_t}$ has escaped the vicinity of $\pi_{\lambda}^\star$.

For brevity, we write 
$\mathcal{F}_n = \sigma(\theta_1, \ldots, \theta_n)$ for the natural filtration of $\theta_n$. 
Now, we are ready to state the next lemma. 

\begin{lemma}\label{lemma: three claims}
 Let $\pi_{\lambda}^*$ be the optimal policy. Then, for all $n\in\{1,2, \ldots\}$, the following statements hold: 
 \begin{enumerate}
     \item $\Omega_{n+1} \subseteq \Omega_{n}$ and $E_{n+1} \subseteq E_{n}$.
     \item $E_{n-1} \subseteq \Omega_{n}$.
     \item Consider the ``large noise'' event
$$
\begin{aligned}
\tilde{E}_{n} & \equiv E_{n-1} \backslash E_{n}=E_{n-1} \cap\left\{R_{n}>\epsilon_0\right\} \\
&=\left\{R_t \leq \epsilon_0 \text { for all } t=1,2, \ldots, n-1 \text { and } R_{n}>\epsilon_0\right\}
\end{aligned}
$$
and let $\tilde{R}_{n}=R_{n} \mathbb{1}_{E_{n-1}}$ denote the cumulative error subject to the noise being ``small'' until time $n$. Then,
\begin{align}\label{eq: third argument}
    \mathbb{E}\left[\tilde{R}_{n}\right] \leq \mathbb{E}\left[\tilde{R}_{n-1}\right]+G^2\sigma^2\eta_n^2+\frac{\eta_n \sigma^2}{4B}-\epsilon_0 \mathbb{P}\left(\tilde{E}_{n-1}\right).
\end{align}
By convention, we write $\tilde{E}_{0}=\varnothing$ and $\tilde{R}_{0}=0$.
 \end{enumerate}
\end{lemma}
\begin{proof}
Statement 1 is obviously true. For Statement 2, we proceed inductively:
\begin{itemize}
    \item For the base case $n=1$, we have $\Omega_{1}=\left\{\pi_{\theta_1} \in \mathcal{U}\right\} \supseteq\left\{\pi_{\theta_1} \in \mathcal{U}_{1}\right\}=\Omega$ because $\pi_{\theta_1}$ is initialized in $\mathcal{U}_{1} \subseteq \mathcal{U}$. Since $E_{0}=\Omega$, our claim follows.
    
    \item For the inductive step, assume that $E_{n-1} \subseteq \Omega_{n}$ for some $n \geq 1$. To show that $E_{n} \subseteq \Omega_{n+1}$, we fix a realization in $E_{n}$ such that $R_t \leq \varepsilon$ for all $t=1,2, \ldots, n$. Since $E_{n} \subseteq E_{n-1}$, the inductive hypothesis posits that $\Omega_{n}$ also occurs, i.e., $\pi_{\theta_t} \in \mathcal{U}$ for all $t=1,2, \ldots, n$; hence, it suffices to show that $\pi_{\theta_{n+1}} \in \mathcal{U}$.
To that end, given that $\pi_{\theta_t} \in {\mathcal{U}}$ for all $t=1,2, \ldots n$, the distance estimate
\eqref{eq: D t+1 and D t} readily gives
$D(\theta_{t+1}) \leq D(\theta_{t})+\eta_t\xi_t+\frac{\eta_t}{4}\norm{e_t}_2^2 $  for all $t=1,2,\ldots,n.
$
Therefore, after telescoping, we obtain
\begin{align*}
D(\theta_{n+1})\leq &D(\theta_{1})+M_n+S_n\leq D(\theta_{1})+\sqrt{R_n}+R_n \\
\leq & \varepsilon+\sqrt{\varepsilon}+\varepsilon\\
= & 2 \varepsilon+\sqrt{\varepsilon}
\end{align*}
by the inductive hypothesis. This completes the induction.
\end{itemize}

For Statement 3, we decompose $\tilde{R}_{n}$ as
$$
\begin{aligned}
\tilde{R}_{n}&=R_{n} \mathbb{1}_{E_{n-1}} \\
&=R_{n-1} \mathbb{1}_{E_{n-1}}+\left(R_{n}-R_{n-1}\right) \mathbb{1}_{E_{n-1}} \\
&=R_{n-1} \mathbb{1}_{E_{n-2}}-R_{n-1} \mathbb{1}_{\tilde{E}_{n-1}}+\left(R_{n}-R_{n-1}\right) \mathbb{1}_{E_{n-1}} \\
&=\tilde{R}_{n-1}+\left(R_{n}-R_{n-1}\right) \mathbb{1}_{E_{n-1}}-R_{n-1} \mathbb{1}_{\tilde{E}_{n-1}}
\end{aligned}
$$
where we have used the fact that $E_{n-1}=E_{n-2} \backslash \tilde{E}_{n-1}$ so $\mathbb{1}_{E_{n-1}}=\mathbb{1}_{E_{n-2}}-\mathbb{1}_{\tilde{E}_{n-1}}$ (recall that $\left.E_{n-1} \subseteq E_{n-2}\right)$. Then, by the definition of $R_{n}$, we have
$$
R_{n}-R_{n-1}=2 M_{n-1} \eta_{n} \xi_{n}+\eta_{n}^{2} \xi_{n}^{2}+\frac{1}{4} \eta_{n}\left\|e_{n}\right\|^{2}
$$
and therefore
\begin{align}\label{eq: first term}
&\mathbb{E}\left[\left(R_{n}-R_{n-1}\right) \mathbb{1}_{E_{n-1}}\right]= \\
\nonumber &2 \eta_{n} \mathbb{E}\left[M_{n-1} \xi_{n} \mathbb{1}_{E_{n-1}}\right] +\eta_{n}^{2} \mathbb{E}\left[\xi_{n}^{2} \mathbb{1}_{E_{n-1}}\right] +\frac{1}{4} \eta_{n} \mathbb{E}\left[\left\|e_{n}\right\|^{2} \mathbb{1}_{E_{n-1}}\right].
\end{align}

However, since $E_{n-1}$ and $M_{n-1}$ are both $\mathcal{F}_{n}$-measurable, we have the following estimates:
\begin{itemize}
    \item For the term in \eqref{eq: first term}, by the unbiasedness of the gradient estimator shown in Lemma \ref{lemma: unbias}, we have:
$
\mathbb{E}\left[M_{n-1} \xi_{n} \mathbb{1}_{E_{n-1}}\right]=\mathbb{E}\left[M_{n-1} \mathbb{1}_{E_{n-1}} \mathbb{E}\left[\xi_{n} \mid \mathcal{F}_{n}\right]\right]=0.
$
\item The second term in \eqref{eq: first term} is where the conditioning on $E_{n-1}$ plays the most important role. It holds that:
$$
\begin{aligned}
\mathbb{E}\left[\xi_{n}^{2} \mathbb{1}_{E_{n-1}}\right] &=\mathbb{E}\left[\mathbb{1}_{E_{n-1}} \mathbb{E}\left[\left\langle e_{n}, \nabla V_\lambda^{\theta_n}(\rho)\right\rangle^{2} \mid \mathcal{F}_{n}\right]\right] \\
& \leq \mathbb{E}\left[\mathbb{1}_{E_{n-1}}\left\|\nabla V_\lambda^{\theta_n}(\rho)\right\|^{2} \mathbb{E}\left[\left\|e_{n}\right\|^{2} \mid \mathcal{F}_{n}\right]\right] \\
& \leq \mathbb{E}\left[\mathbb{1}_{\Omega_{n}}\left\|\nabla V_\lambda^{\theta_n}(\rho)\right\|^{2} \mathbb{E}\left[\left\|e_{n}\right\|^{2} \mid \mathcal{F}_{n}\right]\right] \\
& \leq G^{2} \sigma^{2}
\end{aligned}
$$
where the first inequality is due to the Cauchy-Schwarz inequality, the second inequality follows from $E_{n-1} \subseteq \Omega_{n}$ and the last inequality results from Lemmas \ref{lemma: Max Ent PG} and \ref{lemma: bounded variance}.

\item Finally, for the third term in \eqref{eq: first term}, we have:
\begin{align} \label{eq: third term bound}
  \frac{\eta_n}{4}\mathbb{E}\left[\norm{e_n}_2^2 \mathbb{1}_{E_{n-1}}\right]\leq\frac{\eta_n \sigma^2}{4B}.
\end{align}
\end{itemize}

Thus, putting together all of the above, we obtain
$
\mathbb{E}\left[\left(R_{n}-R_{n-1}\right) \mathbb{1}_{E_{n-1}}\right] \leq G^2\sigma^2\eta_n^2 $ $+\frac{\eta_n \sigma^2}{4B}.
$
Since $R_{n-1}>\varepsilon$ if $\tilde{E}_{n-1}$ occurs, we obtain
$
\mathbb{E}\left[R_{n-1} \mathbb{1}_{\tilde{E}_{n-1}}\right] \geq \varepsilon \mathbb{E}\left[\mathbb{1}_{\tilde{E}_{n-1}}\right]=\varepsilon \mathbb{P}\left(\tilde{E}_{n-1}\right).
$
This completes the proof of Statement 3.
\end{proof}

With the above results, we can show that the cumulative mean square error $R_n$ is small with high probability at all times. 
\begin{lemma}\label{lemma: En high prob}
Consider an arbitrary  tolerance level $\delta>0$. If Algorithm \ref{alg:unbiased PG} is run with a step-size schedule of the form $\eta_t=1 /(t+t_0)$ where $t_0\geq \sqrt{\frac{3 \sigma^2}{2\delta \epsilon_0}}$ and a batch size schedule $B_t\geq \frac{1}{\eta_t}$, we have
$
\mathbb{P}\left(E_{n}\right) \geq 1-\delta/6, \text { for all } n=1,2, \ldots
$
\end{lemma}

\begin{proof}
 We begin by bounding the probability of the ``large noise'' event $\tilde{E}_{n}=E_{n-1} \backslash E_{n}$ as follows:
$$
\begin{aligned}
\mathbb{P}\left(\tilde{E}_{n}\right) &=\mathbb{P}\left(E_{n-1} \backslash E_{n}\right)=\mathbb{P}\left(E_{n-1} \cap\left\{R_{n}>\varepsilon\right\}\right)\\ &=\mathbb{E}\left[\mathbb{1}_{E_{n-1}} \times \mathbb{1}_{\left\{R_{n}>\varepsilon\right\}}\right] \\
& \leq \mathbb{E}\left[\mathbb{1}_{E_{n-1}} \times\left(R_{n} / \varepsilon\right)\right] =\mathbb{E}\left[\tilde{R}_{n}\right] / \varepsilon,
\end{aligned}
$$
which is derived by using the fact that $R_{n} \geq 0$ (so $\left.\mathbb{1}_{\left\{R_{n}>\varepsilon\right\}} \leq R_{n} / \varepsilon\right)$. Now, by summing up \eqref{eq: third argument}, we conclude that 
$
\mathbb{E}\left[\tilde{R}_{n}\right] \leq \mathbb{E}\left[\tilde{R}_{0}\right]+\frac{\sigma^2}{4B} \sum_{t=1}^{n} \eta_t-\varepsilon \sum_{t=1}^{n} \mathbb{P}\left(\tilde{E}_{t-1}\right).
$
Hence, combining the above results, we obtain the estimate
$$
\sum_{t=1}^{n} \mathbb{P}\left(\tilde{E}_{k}\right) \leq \frac{\sigma^2}{4B\epsilon_0} \sum_{t=1}^{n} \eta_t \leq \frac{\sigma^2}{4\epsilon_0} \sum_{t=1}^{n} \eta_t^2 \leq \frac{\sigma^2 \Gamma}{4\epsilon_0},
$$
where $\Gamma=\sum_{t=1}^{\infty} \eta_{t}^{2}=\sum_{t=1}^{\infty}(t+t_0)^{-2}$, and we have used the relations that $\tilde{R}_{0}=0$ and $\tilde{E}_{0}=\varnothing$ (by convention).
By choosing $t_0\geq \sqrt{\frac{3 \sigma^2}{2\delta \epsilon_0}}$, we  ensure that $\frac{\sigma^2 \Gamma}{4\epsilon_0}<\delta/6$; moreover, since the events $\tilde{E}_{t}$ are disjoint for all $t=1,2, \ldots$, we obtain
$
\mathbb{P}\left(\bigcup_{t=1}^{n} \tilde{E}_{t}\right)=\sum_{t=1}^{n} \mathbb{P}\left(\tilde{E}_{t}\right) \leq \delta/6.
$
Hence,
$
\mathbb{P}\left(E_{n}\right)=\mathbb{P}\left(\bigcap_{t=1}^{n} \tilde{E}_{t}^{\mathrm{c}}\right) \geq 1-\delta/6
$
as claimed.
\end{proof}

Furthermore, we can show that the entropy-regularized value function  $V_\lambda^\theta(\rho)$ is locally quadratic around the optimal policy $\pi_{\theta^\ast}$.
\begin{lemma}\label{lemma: local quad of optimal value}
For every policy $\pi_{\theta}$, we have
\begin{align*}
D(\theta)
\geq &\frac{\lambda \min_s\rho(s)}{2 \ln 2}    \left|\pi_\theta({a} \mid {s}) - \pi_{\theta^*}({a} \mid {s})\right|^2, \quad \forall s\in\mathcal{S}, a \in\mathcal{A}.
\end{align*}
\end{lemma}
\begin{proof}
It follows from the soft sub-optimality difference lemma (Lemma 26 in \cite{mei2020global}) that
\begin{align*}
&V_\lambda^{\theta^\ast}(\rho)-V_\lambda^{\theta}(\rho)\\
=&\frac{1}{1-\gamma} \sum_{s}\left[d_{\rho}^{\pi_\theta}(s) \cdot \lambda \cdot D_{\mathrm{KL}}\left(\pi_\theta(\cdot \mid s) \| \pi_{\theta^*}(\cdot \mid s)\right)\right]\\
\geq &\frac{1}{1-\gamma} \sum_{s}\left[d_{\rho}^{\pi_\theta}(s) \cdot \lambda \cdot \frac{1}{2 \ln 2} \|\pi_\theta(\cdot \mid s) - \pi_{\theta^*}(\cdot \mid s)\|_1^2\right]\\
\geq &\frac{\lambda}{2 \ln 2}\sum_{s}\left[{\rho}(s) \cdot   \|\pi_\theta(\cdot \mid s) - \pi_{\theta^*}(\cdot \mid s)\|_1^2\right]\\
\geq &\frac{\lambda}{2 \ln 2}\sum_{s}\left[{\rho}(s) \cdot   \|\pi_\theta(\cdot \mid s) - \pi_{\theta^*}(\cdot \mid s)\|_2^2\right]\\
\geq &\frac{\lambda }{2 \ln 2}\left[   \rho(s)  \|\pi_\theta(\cdot \mid s) - \pi_{\theta^*}(\cdot \mid s)\|_2^2\right]\quad \forall s\in\mathcal{S}\\
\geq &\frac{\lambda \min_s\rho(s)}{2 \ln 2}    \left|\pi_\theta({a} \mid {s}) - \pi_{\theta^*}({a} \mid {s})\right|^2, \quad \forall s\in\mathcal{S}, a \in\mathcal{A},
\end{align*}
where the first inequality is due to  Theorem 11.6 in \cite{cover1999elements} stating that
\begin{equation*}
D_\text{KL}\left[P(\cdot) \mid Q(\cdot) \right] \geq \frac{1}{2 \ln 2} \|P(\cdot)-Q(\cdot)\|_1^2
\end{equation*}
for every two discrete distributions $P(\cdot)$ and $Q(\cdot)$. Moreover, the second inequality is due to $d_{\rho}^{\pi_\theta}(s) \geq (1-\gamma) \rho(s)$ and the third inequality is due to the equivalence between $\ell_1$-norm and $\ell_2$-norm. 
This completes the proof.
\end{proof}

\subsection{Proof of Lemma \ref{lemma: pi lower bound high prob}}
Since the sequence $\Omega_{n}$ is decreasing and $\Omega_{n} \supseteq E_{n-1}$ (by the second part of Lemma \ref{lemma: three claims}), Lemma \ref{lemma: En high prob} yields that
$
\mathbb{P}\left(\Omega_{T_2}\right)\geq\inf _{n} \mathbb{P}\left(\Omega_{n}\right) \geq \inf _{n} \mathbb{P}\left(E_{n-1}\right) \geq 1-\delta/6
$
provided that $t_0$ is chosen large enough. 

Now, it remains to show that $\Omega_{T_2}\subseteq\Omega_{\alpha,1}^{T_2}$. We fix a realization in $\Omega_{T_2}$ such that $D(\theta_{t})\leq 2\epsilon_0+\sqrt{\epsilon_0}$ for all $t=1,2,\ldots, T_2$.
By Lemma \ref{lemma: local quad of optimal value}, we have
\begin{align*}
& \left|\pi_{\theta_t}({a} \mid {s}) - \pi_{{\theta^*}}({a} \mid {s})\right| \\
\leq& \sqrt{\frac{2  D({\theta_t}) \ln 2}{\lambda \min_s\rho(s)}}\leq\sqrt{\frac{2(2\epsilon_0+\sqrt{\epsilon_0})\ln 2}{\lambda \min_s\rho(s)}}\\
\leq& \sqrt{\frac{6\sqrt{\epsilon_0}\ln 2}{\lambda \min_s\rho(s)}} \leq \alpha  \exp(\frac{-\bar{r}}{(1-\gamma)\lambda})
  \leq \alpha \min_{s,a}\pi_{{\theta^*}}({a} \mid {s}),
\end{align*}
where the second inequality is due to the condition that the event $\Omega_{T_2}$ occurs, the third inequality is due to $\epsilon_0\leq\sqrt{\epsilon_0}$ when $\epsilon_0\leq 1$, the forth inequality is due to the definition of $\epsilon_0$, and the last inequality is  due to Theorem 1 in \cite{nachum2017bridging} where it holds that
$\log \pi_{\lambda}^{*}(a \mid s)=\frac{1}{\lambda} \left( {Q}^{\pi_{\lambda}^{*}}(s, a)-{V}^{\pi_{\lambda}^{*}}(s) \right) \geq \frac{-\bar{r}}{(1-\gamma)\lambda}, \ \forall(s, a) \in \mathcal{S} \times \mathcal{A}.$

Now, it can be easily verified that
$\pi_{\theta_t}({a} \mid {s})\geq \pi_{{\theta^*}}({a} \mid {s})  -\alpha\min_{s,a}\pi_{{\theta^*}}({a} \mid {s}).$
For every $t\in \{1,2,\ldots,T_2\}$, let $\bar{s},\bar{a}=\argmin_{s,a} \pi_{\theta_t}({a} \mid {s})$. One can write 
\begin{align*}
\min_{s,a}\pi_{\theta_t}({a} \mid {s}) =& \pi_{\theta_t}(\bar{a} \mid \bar{s})
\geq \pi_{{\theta^*}}(\bar{a} \mid \bar{s})  -\alpha\min_{s,a}\pi_{{\theta^*}}({a} \mid {s}) \\
\geq& (1-\alpha)\min_{s,a}\pi_{{\theta^*}}(\bar{a} \mid \bar{s}),
\end{align*}
where the last inequality is due to $\pi(a|s)\geq \min_{s,a}\pi(a|s)$ for every $s\in\mathcal{S}$ and $a\in\mathcal{A}$. Thus, we obtain
$\mathbb{P}\left(\Omega_{\alpha, 1}^{T_2}\right)\geq\mathbb{P}\left(\Omega_{T_2}\right) \geq 1-\delta/6$. This completes the proof.

\section{Proof of Theorem \ref{thm: two phases bound}} \label{sec: proof of thm 1}
From Lemma \ref{thm: high complexity global convergence}, we conclude that, with a large batch size, the iterations will converge to a neighborhood of the optimal solution with high probability.
From Lemma \ref{lemma: pi lower bound high prob}, we know that,  with a good initialization, the policies will remain in the interior of the probability simplex with high probability. By combining the above two results, we are now ready to prove the sample complexity of the stochastic PG for entropy-regularized RL.


From Lemma \ref{thm: high complexity global convergence}, we can conclude that $\mathbb{P}(D(\theta_{T_1}) \leq \epsilon_0)\geq 1-\delta$ after the first phase. We then establish the algorithm’s 
sample complexity when the initial policy of the second phase satisfies the good initialization condition $\mathbb{P}(D(\theta_{T_1}) \leq \epsilon_0)\geq 1-\delta$.
It follows from Lemma \ref{lemma: D t+1 and D t} that
\begin{align*} 
&D(\theta_{t+1}) \mathbb{1}_{\Omega_{\alpha,T_1}^t}\\
\leq & \left(1-\frac{\eta_t C(\theta_t)}{4} \right)D(\theta_{t})\mathbb{1}_{\Omega_{\alpha,T_1}^t}-
\frac{\eta_t}{2}\xi_t\mathbb{1}_{\Omega_{\alpha,T_1}^t}+\frac{\eta_t}{4}\norm{e_t}_2^2\mathbb{1}_{\Omega_{\alpha,T_1}^t},
\end{align*}
for all $t\geq T_1$, where $\xi_t=\inner{e_t}{{\nabla}V^{\theta_t}_\lambda(\rho)}$ and $\Omega_{\alpha,T_1}^t$ is defined in \eqref{eq: definition of Omega alpha}. When the event $\Omega_{\alpha,T_1}^t$ occurs, we have $C(\theta_t)\geq C_\alpha$, where $C_\alpha$ is defined in \eqref{eq: C alpha}.
By taking the expectation, we have
\begin{align*}
&\EE \left[ -\frac{\eta_t}{2}\xi_t\mathbb{1}_{\Omega_{\alpha,T_1}^t}+\frac{\eta_t}{4}\norm{e_t}_2^2\mathbb{1}_{\Omega_{\alpha,T_1}^t} \right]\\
=&\EE \left[\mathbb{1}_{\Omega_{\alpha,T_1}^t}\EE \left[ -\frac{\eta_t}{2}\xi_t+\frac{\eta_t}{4}\norm{e_t}_2^2 \Big\mid \mathcal{F}_t\right]\right]\\
=&\EE \left[\mathbb{1}_{\Omega_{\alpha,T_1}^t}\EE \left[ \frac{\eta_t}{4}\norm{e_t}_2^2 \Big\mid \mathcal{F}_t\right]\right]
\leq  \frac{\eta_t \sigma^2}{4B},
\end{align*}
where the first equality is because $\Omega_{\alpha,T_1}^t$ is deterministic conditioning on $\mathcal{F}_t$, 
the second equality is due to the unbiasedness of $\xi_t$ conditioning on $\mathcal{F}_t$, and the first inequality is due to \eqref{eq: third term bound}.
Therefore,
$\EE[ D(\theta_{t+1})\mathbb{1}_{\Omega_{\alpha,T_1}^t}]
\leq \left(1-\frac{\eta_t C_\alpha}{4}\right)\EE\left[ D(\theta_{t})\mathbb{1}_{\Omega_{\alpha,T_1}^t}\right]+\frac{\eta_t \sigma^2}{4B}.$
Arguing inductively yields that
\begin{align*}
&\EE[D(\theta_{T+1})\mathbb{1}_{\Omega_{\alpha,T_1}^T}] \\
\leq & \prod_{i=1}^{T_2} \left(1-\frac{\eta_{T_1+i} C_\alpha}{4}\right) D(\theta_{T_1}) +\sum_{i=1}^{T_2}\left(1-\frac{\eta_{T_1+i} C_\alpha}{4}\right)^i\frac{\eta_{T_1+i} \sigma^2}{4B} \\
\leq & \prod_{i=1}^{T_2} \left(1-\frac{\eta_{T_1+i} C_\alpha}{4}\right) D(\theta_{T_1}) +\sum_{i=1}^{T_2}\frac{\eta_{T_1+i} \sigma^2}{4B}.
\end{align*}
By taking $\eta_{T_1+i}=\frac{4}{C_\alpha (i+t_0)}$, we obtain that
\begin{align*}
\EE[D(\theta_{T+1})\mathbb{1}_{\Omega^T_{\alpha, {T_1}}}]\leq  & \prod_{i=1}^{T_2} \left(\frac{i+t_0-1}{i+t_0}\right) D(\theta_{T_1}) +\frac{\sigma^2}{C_\alpha  B}\sum_{i=1}^{T_2}\frac{1}{i+t_0}\\
\leq &\frac{t_0}{T_2+t_0} D(\theta_{T_1}) +\frac{\sigma^2\ln{(T_2+t_0)}}{B C_\alpha}.
\end{align*}


By the law of total probability and the Markov inequality, we obtain that
\begin{align*}
& \mathbb{P}(D(\theta_{T+1}) \geq \epsilon) \\
=  &    \mathbb{P}(D(\theta_{T+1}) \geq \epsilon, {\Omega_{\alpha,T_1}^T})+\mathbb{P}(D(\theta_{T+1}) \geq \epsilon , \left(\Omega_{\alpha,T_1}^T\right)^c)\\
=  &    \mathbb{P}(D(\theta_{T+1}) \geq \epsilon \mid {\Omega_{\alpha,T_1}^T})\mathbb{P}( \Omega_{\alpha,T_1}^T)\\
&+\mathbb{P}(D(\theta_{T+1}) \geq \epsilon \mid \left(\Omega_{\alpha,T_1}^T\right)^c)\mathbb{P}( \left(\Omega_{\alpha,T_1}^T\right)^c)\\
\leq & \frac{\EE[D(\theta_{T+1})\mid {\Omega_{\alpha,T_1}^T}]}{\epsilon}\mathbb{P}( \Omega_{\alpha,T_1}^T)\\
&+\mathbb{P}(D(\theta_{T+1}) \geq \epsilon \mathbb{1}_{\Omega_{\alpha,T_1}^T}^c)\mathbb{P}( \left(\Omega_{\alpha,T_1}^T\right)^c)\\
\leq & \frac{\EE[D(\theta_{T+1})\mathbb{1}_{\Omega_{\alpha,T_1}^T]}}{\epsilon}+\delta/6\\
\leq &\frac{t_0}{(T_2 +t_0)\epsilon} D(\theta_{T_1}) +\frac{\sigma^2\ln{(T_2+t_0)}}{B C_\alpha \epsilon}+\delta/6,
\end{align*}
where the second inequality follows from Lemma \ref{lemma: pi lower bound high prob}.
To guarantee  $\mathbb{P}(D(\theta_{T+1}) \geq \epsilon) \leq \frac{\delta}{2}$, it suffices to have
$T_2 = \frac{t_0D(\theta_{T_1})}{6\delta\epsilon}-t_0, B= \frac{\sigma^2\ln(T_2 +t_0)}{6 C_\alpha \delta  \epsilon}.$
This completes the proof.
\section{Conclusion}\label{sec:conclu}
In this work, we studied the global convergence and the sample complexity of stochastic PG methods for the entropy-regularized RL with the soft-max parameterization. We proposed two new (nearly) unbiased PG estimators for the entropy-regularized RL and proved that they have a bounded variance even though they could be unbounded. In addition, we developed a two-phase stochastic PG algorithm to overcome the non-coercive landscape challenge.  This work provided the first global convergence result for stochastic PG methods for the entropy-regularized RL and obtained the sample complexity of $\widetilde{\mathcal{O}}(\frac{1}{\epsilon^2})$, where $\epsilon$ is the optimality threshold. {This work paves the way for a deeper understanding of other stochastic  PG methods with entropy-related regularization, including those with trajectory-level KL regularization and policy reparameterization.}
\myred{An important future direction is to study the dependence of the sample complexity of the entropy-regularized RL with respect to the dimension of the state space  and improve the bound. }

\section*{Acknowledgment}
This work was funded by grants from AFOSR, ARO, ONR, NSF and C3.ai Digital Transformation Institute.
\ifCLASSOPTIONcaptionsoff
  \newpage
\fi



%

\bibliographystyle{IEEEtran}
\bibliography{IEEEabrv,ref}

\begin{thebibliography}{10}
\providecommand{\url}[1]{#1}
\csname url@samestyle\endcsname
\providecommand{\newblock}{\relax}
\providecommand{\bibinfo}[2]{#2}
\providecommand{\BIBentrySTDinterwordspacing}{\spaceskip=0pt\relax}
\providecommand{\BIBentryALTinterwordstretchfactor}{4}
\providecommand{\BIBentryALTinterwordspacing}{\spaceskip=\fontdimen2\font plus
\BIBentryALTinterwordstretchfactor\fontdimen3\font minus
  \fontdimen4\font\relax}
\providecommand{\BIBforeignlanguage}[2]{{%
\expandafter\ifx\csname l@#1\endcsname\relax
\typeout{** WARNING: IEEEtran.bst: No hyphenation pattern has been}%
\typeout{** loaded for the language `#1'. Using the pattern for}%
\typeout{** the default language instead.}%
\else
\language=\csname l@#1\endcsname
\fi
#2}}
\providecommand{\BIBdecl}{\relax}
\BIBdecl

\bibitem{williams1991function}
R.~J. Williams and J.~Peng, ``Function optimization using connectionist
  reinforcement learning algorithms,'' \emph{Connection Science}, vol.~3,
  no.~3, pp. 241--268, 1991.

\bibitem{mnih2016asynchronous}
V.~Mnih, A.~P. Badia, M.~Mirza, A.~Graves, T.~Lillicrap, T.~Harley, D.~Silver,
  and K.~Kavukcuoglu, ``Asynchronous methods for deep reinforcement learning,''
  in \emph{International conference on machine learning}.\hskip 1em plus 0.5em
  minus 0.4em\relax PMLR, 2016, pp. 1928--1937.

\bibitem{haarnoja2018soft}
T.~Haarnoja, A.~Zhou, P.~Abbeel, and S.~Levine, ``Soft actor-critic: Off-policy
  maximum entropy deep reinforcement learning with a stochastic actor,'' in
  \emph{International conference on machine learning}.\hskip 1em plus 0.5em
  minus 0.4em\relax PMLR, 2018, pp. 1861--1870.

\bibitem{o2016combining}
B.~O'Donoghue, R.~Munos, K.~Kavukcuoglu, and V.~Mnih, ``Combining policy
  gradient and {Q}-learning,'' \emph{arXiv preprint arXiv:1611.01626}, 2016.

\bibitem{haarnoja2017reinforcement}
T.~Haarnoja, H.~Tang, P.~Abbeel, and S.~Levine, ``Reinforcement learning with
  deep energy-based policies,'' in \emph{International Conference on Machine
  Learning}.\hskip 1em plus 0.5em minus 0.4em\relax PMLR, 2017, pp. 1352--1361.

\bibitem{zang2020teac}
H.~Zang, X.~Li, L.~Zhang, P.~Zhao, and M.~Wang, ``Teac: Intergrating trust
  region and max entropy actor critic for continuous control,''
  \emph{https://openreview.net/references/pdf?id=bzTQQZQ6ix}, 2020.

\bibitem{ziebart2010modeling}
B.~D. Ziebart, \emph{Modeling purposeful adaptive behavior with the principle
  of maximum causal entropy}.\hskip 1em plus 0.5em minus 0.4em\relax Carnegie
  Mellon University, 2010.

\bibitem{schulman2017equivalence}
J.~Schulman, P.~Abbeel, and X.~Chen, ``Equivalence between policy gradients and
  soft {Q}-learning,'' \emph{CoRR}, vol. abs/1704.06440, 2017.

\bibitem{mei2020global}
J.~Mei, C.~Xiao, C.~Szepesvari, and D.~Schuurmans, ``On the global convergence
  rates of softmax policy gradient methods,'' in \emph{International Conference
  on Machine Learning}.\hskip 1em plus 0.5em minus 0.4em\relax PMLR, 2020, pp.
  6820--6829.

\bibitem{lan2021policy}
G.~Lan, ``Policy mirror descent for reinforcement learning: Linear convergence,
  new sampling complexity, and generalized problem classes,''
  \emph{Mathematical programming}, pp. 1--48, 2022.

\bibitem{cen2020fast}
S.~Cen, C.~Cheng, Y.~Chen, Y.~Wei, and Y.~Chi, ``Fast global convergence of
  natural policy gradient methods with entropy regularization,''
  \emph{Operations Research}, 2021.

\bibitem{chung2021beyond}
W.~Chung, V.~Thomas, M.~C. Machado, and N.~Le~Roux, ``Beyond variance
  reduction: Understanding the true impact of baselines on policy
  optimization,'' in \emph{International Conference on Machine Learning}.\hskip
  1em plus 0.5em minus 0.4em\relax PMLR, 2021, pp. 1999--2009.

\bibitem{mei2021understanding}
J.~Mei, B.~Dai, C.~Xiao, C.~Szepesvari, and D.~Schuurmans, ``Understanding the
  effect of stochasticity in policy optimization,'' \emph{Advances in Neural
  Information Processing Systems}, vol.~34, 2021.

\bibitem{williams1992simple}
R.~J. Williams, ``Simple statistical gradient-following algorithms for
  connectionist reinforcement learning,'' \emph{Machine learning}, vol.~8, no.
  3-4, pp. 229--256, 1992.

\bibitem{ahmed2019understanding}
Z.~Ahmed, N.~Le~Roux, M.~Norouzi, and D.~Schuurmans, ``Understanding the impact
  of entropy on policy optimization,'' in \emph{International Conference on
  Machine Learning}.\hskip 1em plus 0.5em minus 0.4em\relax PMLR, 2019, pp.
  151--160.

\bibitem{agarwal2020optimality}
A.~Agarwal, S.~M. Kakade, J.~D. Lee, and G.~Mahajan, ``On the theory of policy
  gradient methods: Optimality, approximation, and distribution shift.''
  \emph{J. Mach. Learn. Res.}, vol.~22, no.~98, pp. 1--76, 2021.

\bibitem{xiao2022convergence}
L.~Xiao, ``On the convergence rates of policy gradient methods,'' \emph{arXiv
  preprint arXiv:2201.07443}, 2022.

\bibitem{shani2020adaptive}
L.~Shani, Y.~Efroni, and S.~Mannor, ``Adaptive trust region policy
  optimization: Global convergence and faster rates for regularized {MDP}s,''
  in \emph{Proceedings of the AAAI Conference on Artificial Intelligence},
  vol.~34, 2020, pp. 5668--5675.

\bibitem{bhandari2019global}
J.~Bhandari and D.~Russo, ``Global optimality guarantees for policy gradient
  methods,'' \emph{arXiv preprint arXiv:1906.01786}, 2019.

\bibitem{zhang2021convergence}
J.~Zhang, C.~Ni, C.~Szepesvari, M.~Wang \emph{et~al.}, ``On the convergence and
  sample efficiency of variance-reduced policy gradient method,''
  \emph{Advances in Neural Information Processing Systems}, vol.~34, pp.
  2228--2240, 2021.

\bibitem{NEURIPS2020_30ee748d}
J.~Zhang, A.~Koppel, A.~S. Bedi, C.~Szepesvari, and M.~Wang, ``Variational
  policy gradient method for reinforcement learning with general utilities,''
  in \emph{Advances in Neural Information Processing Systems}, H.~Larochelle,
  M.~Ranzato, R.~Hadsell, M.~F. Balcan, and H.~Lin, Eds., vol.~33.\hskip 1em
  plus 0.5em minus 0.4em\relax Curran Associates, Inc., 2020, pp. 4572--4583.

\bibitem{li2021softmax}
G.~Li, Y.~Wei, Y.~Chi, Y.~Gu, and Y.~Chen, ``Softmax policy gradient methods
  can take exponential time to converge,'' in \emph{Conference on Learning
  Theory}.\hskip 1em plus 0.5em minus 0.4em\relax PMLR, 2021, pp. 3107--3110.

\bibitem{mei2021leveraging}
J.~Mei, Y.~Gao, B.~Dai, C.~Szepesvari, and D.~Schuurmans, ``Leveraging
  non-uniformity in first-order non-convex optimization,'' \emph{International
  Conference on Machine Learning}, 2021.

\bibitem{liu2020improved}
Y.~Liu, K.~Zhang, T.~Basar, and W.~Yin, ``An improved analysis of
  (variance-reduced) policy gradient and natural policy gradient methods,''
  \emph{Advances in Neural Information Processing Systems}, vol.~33, 2020.

\bibitem{ding2021global}
Y.~Ding, J.~Zhang, and J.~Lavaei, ``On the global optimum convergence of
  momentum-based policy gradient,'' in \emph{International Conference on
  Artificial Intelligence and Statistics}.\hskip 1em plus 0.5em minus
  0.4em\relax PMLR, 2022, pp. 1910--1934.

\bibitem{eysenbach2019if}
B.~Eysenbach and S.~Levine, ``If {MaxEnt} {RL} is the answer, what is the
  question?'' \emph{arXiv preprint arXiv:1910.01913}, 2019.

\bibitem{sutton2018reinforcement}
R.~S. Sutton and A.~G. Barto, \emph{Reinforcement learning: An
  introduction}.\hskip 1em plus 0.5em minus 0.4em\relax MIT press, 2018.

\bibitem{cayci2021linear}
S.~Cayci, N.~He, and R.~Srikant, ``Linear convergence of entropy-regularized
  natural policy gradient with linear function approximation,'' in
  \emph{International conference on machine learning}.\hskip 1em plus 0.5em
  minus 0.4em\relax PMLR, 2021.

\bibitem{zhang2020global}
K.~Zhang, A.~Koppel, H.~Zhu, and T.~Basar, ``Global convergence of policy
  gradient methods to (almost) locally optimal policies,'' \emph{SIAM Journal
  on Control and Optimization}, vol.~58, no.~6, pp. 3586--3612, 2020.

\bibitem{sutton1999policy}
R.~S. Sutton, D.~A. McAllester, S.~P. Singh, Y.~Mansour \emph{et~al.}, ``Policy
  gradient methods for reinforcement learning with function approximation.'' in
  \emph{NIPs}, vol.~99.\hskip 1em plus 0.5em minus 0.4em\relax Citeseer, 1999,
  pp. 1057--1063.

\bibitem{baxter2001infinite}
J.~Baxter and P.~L. Bartlett, ``Infinite-horizon policy-gradient estimation,''
  \emph{Journal of Artificial Intelligence Research}, vol.~15, pp. 319--350,
  2001.

\bibitem{benaim1999dynamics}
M.~Bena{\"\i}m, ``Dynamics of stochastic approximation algorithms,'' in
  \emph{Seminaire de probabilites XXXIII}.\hskip 1em plus 0.5em minus
  0.4em\relax Springer, 1999, pp. 1--68.

\bibitem{borkar2009stochastic}
V.~S. Borkar, \emph{Stochastic approximation: a dynamical systems
  viewpoint}.\hskip 1em plus 0.5em minus 0.4em\relax Springer, 2009, vol.~48.

\bibitem{bertsekas2000gradient}
D.~P. Bertsekas and J.~N. Tsitsiklis, ``Gradient convergence in gradient
  methods with errors,'' \emph{SIAM Journal on Optimization}, vol.~10, no.~3,
  pp. 627--642, 2000.

\bibitem{benaim1996asymptotic}
M.~Bena{\"\i}m and M.~W. Hirsch, ``Asymptotic pseudotrajectories and chain
  recurrent flows, with applications,'' \emph{Journal of Dynamics and
  Differential Equations}, vol.~8, no.~1, pp. 141--176, 1996.

\bibitem{kushner2012stochastic}
H.~J. Kushner and D.~S. Clark, \emph{Stochastic approximation methods for
  constrained and unconstrained systems}.\hskip 1em plus 0.5em minus
  0.4em\relax Springer Science \& Business Media, 2012, vol.~26.

\bibitem{mertikopoulos2020almost}
P.~Mertikopoulos, N.~Hallak, A.~Kavis, and V.~Cevher, ``On the almost sure
  convergence of stochastic gradient descent in non-convex problems,''
  \emph{Advances in Neural Information Processing Systems}, vol.~33, pp.
  1117--1128, 2020.

\bibitem{zhang2020sample}
J.~Zhang, J.~Kim, B.~O'Donoghue, and S.~Boyd, ``Sample efficient reinforcement
  learning with {REINFORCE},'' \emph{35th AAAI Conference on Artificial
  Intelligence}, 2021.

\bibitem{mei2020escaping}
J.~Mei, C.~Xiao, B.~Dai, L.~Li, C.~Szepesv{\'a}ri, and D.~Schuurmans,
  ``Escaping the gravitational pull of softmax,'' \emph{Advances in Neural
  Information Processing Systems}, vol.~33, pp. 21\,130--21\,140, 2020.

\bibitem{puterman2014markov}
M.~L. Puterman, \emph{Markov decision processes: discrete stochastic dynamic
  programming}.\hskip 1em plus 0.5em minus 0.4em\relax John Wiley \& Sons,
  2014.

\bibitem{hsieh2019convergence}
Y.-G. Hsieh, F.~Iutzeler, J.~Malick, and P.~Mertikopoulos, ``On the convergence
  of single-call stochastic extra-gradient methods,'' \emph{Advances in Neural
  Information Processing Systems}, vol.~32, 2019.

\bibitem{hsieh2020explore}
------, ``Explore aggressively, update conservatively: Stochastic extragradient
  methods with variable stepsize scaling,'' \emph{Advances in Neural
  Information Processing Systems}, vol.~33, pp. 16\,223--16\,234, 2020.

\bibitem{mertikopoulos2019learning}
P.~Mertikopoulos and Z.~Zhou, ``Learning in games with continuous action sets
  and unknown payoff functions,'' \emph{Mathematical Programming}, vol. 173,
  no.~1, pp. 465--507, 2019.

\bibitem{cover1999elements}
T.~M. Cover, \emph{Elements of information theory}.\hskip 1em plus 0.5em minus
  0.4em\relax John Wiley \& Sons, 1999.

\bibitem{nachum2017bridging}
O.~Nachum, M.~Norouzi, K.~Xu, and D.~Schuurmans, ``Bridging the gap between
  value and policy based reinforcement learning,'' \emph{Advances in neural
  information processing systems}, vol.~30, 2017.

\bibitem{dingbeyond}
Y.~Ding, J.~Zhang, and J.~Lavaei, ``Beyond exact gradients: Convergence of
  stochastic soft-max policy gradient methods with entropy regularization,''
  \emph{https://lavaei.ieor.berkeley.edu/Entropy\_2021\_1.pdf}, 2022.

\bibitem{yuan2021general}
R.~Yuan, R.~M. Gower, and A.~Lazaric, ``A general sample complexity analysis of
  vanilla policy gradient,'' in \emph{International Conference on Artificial
  Intelligence and Statistics}.\hskip 1em plus 0.5em minus 0.4em\relax PMLR,
  2022, pp. 3332--3380.

\end{thebibliography}

\appendices
\section{Properties of stochastic policy gradient}\label{sec: appex_SPG}

\subsection{Proof of Lemmas \ref{lemma: Max Ent PG}, \ref{lemma: unbias}, \ref{lemma: bounded variance} and \ref{lemma:bias of truncation}}
\begin{proof}
Due to the space restriction, please refer to the online version of the paper \cite{dingbeyond} for the complete proofs.
\end{proof}

\subsection{Proof of Lemma \ref{lemma:bounded variance of truncation}}
\begin{proof}
\myred{\textbf{Step 1: Decomposition of the variance:}}
For the simplicity of the notation, we first define:
\begin{align} 
  &g_1(\tau^H|\theta,\rho)  \label{eq: PGT 1} =\sum_{h=0}^{H-1}\left(\sum_{j=0}^{h} \nabla \log \pi_\theta(a_j|s_j)\right) \gamma^h r_h(s_h,a_h)\\
 &g_2(\tau^H|\theta,\rho)  \label{eq: PGT 2} =\lambda\sum_{h=0}^{H-1}\left(\sum_{j=0}^{h} \nabla \log \pi_\theta(a_j|s_j)\right) \left(-\gamma^h \log \pi_\theta(a_h|s_h) \right).
\end{align}

By the definition of the variance and the Cauchy-Schwarz inequality, we have
\begin{align} \label{eq: variance of g}
& \nonumber \text{Var}(\hat{\nabla} V^{\theta,H}_\lambda(\rho))\\
=&\EE\left[\left(g_1(\tau^H|\theta,\rho) +g_2(\tau^H|\theta,\rho)) \right.\right. \\
\nonumber &\left.\left.\hspace{1.5cm}-\EE[g_1(\tau^H|\theta,\rho)] -\EE[g_2(\tau^H|\theta,\rho)] \right)^2\right]\\
\nonumber \leq &3\EE[(g_1(\tau^H|\theta,\rho)-\EE[g_1(\tau^H|\theta,\rho)])^2]\\
& \hspace{1.5cm}+3 \EE[(g_2(\tau^H|\theta,\rho))- \EE[g_2(\tau^H|\theta,\rho))])^2]\\
= &3 \left(\text{Var}(g_1(\tau^H|\theta,\rho))+\text{Var}(g_2(\tau^H|\theta,\rho))\right).
\end{align}

\myred{\textbf{Step 2: Bounded variance of $g_1$:}} As shown in Lemma 4.2 of \cite{yuan2021general}, the fact that $\norm{\nabla \log\pi_\theta(a|s)}_2\leq 2$ for all $\theta \in \mathbb{R}^{|\mathcal{S}||\mathcal{A}|}$  directly implies that $\text{Var}(g_1(\tau^H|\theta,\rho)) \leq \frac{4 \bar{r}^2}{(1-\gamma)^4}$ for all $\theta \in \mathbb{R}^{|\mathcal{S}||\mathcal{A}|}$. 


\myred{\textbf{Step 3: Bounded variance of $g_2$:}}  Then, it remains to prove the bounded variance of $g_2$. Firstly, it can be observed that
\begin{align*}
\norm{g_2}= & \lambda \norm{\sum_{h=0}^{H-1}\left(\sum_{j=0}^{h} \nabla \log \pi_\theta(a_j^i|s_j^i)\right) \left(-\gamma^h \log \pi_\theta(a^i_h|s^i_h) \right) }\\
\leq & - \lambda\sum_{h=0}^{H-1} \left(\sum_{j=0}^{h} \norm{\nabla \log \pi_\theta(a_j^i|s_j^i)}\right) \gamma^h \log \pi_\theta(a^i_h|s^i_h)  \\
\leq & -2 \lambda\sum_{h=0}^{H-1}(h+1) \gamma^h \log \pi_\theta(a^i_h|s^i_h).
\end{align*}
where the first inequality is due to the triangle inequality and the second inequality is due to $\norm{\nabla \log \pi_\theta(a_j^i|s_j^i)} \leq 2$. Then, by taking the squre of $\norm{g_2}$, we obtain 
\begin{align*}
   \norm{g_2}^2\leq &4 \lambda^2\left(\sum_{h=0}^{H-1}(h+1) \gamma^h \log \pi_\theta(a^i_h|s^i_h)\right)^2\\
   = &4 \lambda^2 \left(\sum_{h=0}^{H-1}(h+1) \sqrt{\gamma^h}\sqrt{\gamma^h} \log \pi_\theta(a^i_h|s^i_h)\right)^2\\
   \leq &4 \lambda^2\left(\sum_{h=0}^{H-1}(h+1)^2 \gamma^h\right) \left(\sum_{h=0}^{H-1} \gamma^h \left(\log\pi_\theta(a^i_h|s^i_h)\right)^2\right)\\
   = &4 \lambda^2\left(\sum_{h=0}^{H-1}(h^2+2h+1) \gamma^h\right) \left(\sum_{h=0}^{H-1} \gamma^h \left(\log\pi_\theta(a^i_h|s^i_h)\right)^2\right) \\
   \leq &4 \lambda^2\left(\frac{\gamma^2+\gamma}{(1-\gamma)^3} +\frac{2\gamma}{(1-\gamma)^2 }+\frac{1}{1-\gamma}\right)\\
   &\hspace{3cm}\left(\sum_{h=0}^{H-1} \gamma^h \left(\log\pi_\theta(a^i_h|s^i_h)\right)^2\right)  \\
=& 4 \lambda^2\left(\frac{\gamma+1}{(1-\gamma)^3}\right) \left(\sum_{h=0}^{H-1} \gamma^h \left(\log\pi_\theta(a^i_h|s^i_h)\right)^2\right) 
\end{align*}
where the second inequality is due to the Cauchy-Schwarz inequality and the last inequality is due to $\sum_{h=0}^\infty h^2 \gamma^h = \frac{\gamma^2+\gamma}{(1-\gamma)^3}$, $\sum_{h=0}^\infty h \gamma^h = \frac{\gamma}{(1-\gamma)^2}$ and $\sum_{h=0}^\infty  \gamma^h = \frac{1}{1-\gamma}$.

By taking expectation of $g_2$ over the sample trajectory $\tau^H$, it holds that
\begin{align} \label{eq: expectation of g2 square 2}
\nonumber &\EE_{\tau^H \sim p(\tau^H|\theta)}\left[ \norm{g_2}^2\right]\\
\leq &4 \lambda^2 \left(\frac{\gamma+1}{(1-\gamma)^3}\right) \sum_{h=0}^{H-1} \gamma^h \EE_{\tau^H \sim p(\tau^H|\theta)}\left[ \left(\log\pi_\theta(a^i_h|s^i_h)\right)^2\right].
\end{align}
Since the realizations of $a_h^i$ and $s_h^i$ do not depend on the randomness in $s_{h+1}, a_{h+1}, \ldots,$ $s_H$, we have
\begin{align*}
& \EE_{\tau^H \sim p(\tau^H |\theta)}\left[ \left(\log\pi_\theta(a^i_h|s^i_h)\right)^2\right]
\end{align*}
\begin{align*}
=&\EE_{s_0 \sim \rho, a_0 \sim \pi_\theta(\cdot|s_0), s_1\sim p(\cdot|a_0,s_0) \ldots a_{H-1} \sim \pi_\theta(\cdot|s_{H-1}), s_{H} \sim p(\cdot|s_{H-1},a_{H-1})} \\
&\left[ \left(\log\pi_\theta(a^i_h|s^i_h)\right)^2\right]\\
=& \EE_{s_0 \sim \rho, a_0 \sim \pi_\theta(\cdot|s_0), s_1\sim p(\cdot|a_0,s_0) \ldots a_h \sim \pi_\theta(\cdot|s_h)}\left[ \left(\log\pi_\theta(a^i_h|s^i_h)\right)^2\right] \\ 
=& \EE_{s_0 \sim \rho, a_0 \sim \pi_\theta(\cdot|s_0), s_1\sim p(\cdot|a_0,s_0) \ldots s_{h} \sim p(\cdot|a_{h-1},s_{h-1})}\\
&\left[ \sum_{a_h \in \mathcal{A}}  \pi_\theta(a_h|s_h)\left(\log\pi_\theta(a^i_h|s^i_h)\right)^2\right].
\end{align*}

Since the maximizer for the constrained problem 
\begin{align*} 
    \max \ \sum_{i=1}^n x_i (\log{x_i})^2 \quad \text{such that }  \sum_{i=1}^n x_i=1,
\end{align*}
is $x_1=x_2=\cdots=x_n=\frac{1}{n}$ and the maximum solution is $(\log{n})^2$. Thus, we have $ \sum_{a_h \in \mathcal{A}}  \pi_\theta(a_h|s_h)\left(\log\pi_\theta(a^i_h|s^i_h)\right)^2 \leq (\log |\mathcal{A}|)^2$ and 
\begin{align} \label{eq: expectation of log po square}
&\nonumber\EE_{\tau^H \sim p(\tau^H|\theta)}\left[ \left(\log\pi_\theta(a^i_h|s^i_h)\right)^2\right]\\
\nonumber \leq & \EE_{s_0 \sim \rho, a_0 \sim \pi_\theta(\cdot|s_0), s_1\sim p(\cdot|a_0,s_0) \ldots s_{H} \sim p(\cdot|a_{H-1},s_{H-1})}\left[(\log |\mathcal{A}|)^2\right]\\
=&(\log |\mathcal{A}|)^2.
\end{align}

By combining \eqref{eq: expectation of g2 square 2} and \eqref{eq: expectation of log po square}, we have
\begin{align*}
\text{Var}(g_2)\leq &\EE_{\tau^H \sim p(\tau^H|\theta)}\left[ \norm{g_2}^2\right]\\
\leq &4 \lambda^2 \left(\frac{\gamma+1}{(1-\gamma)^3}\right) \sum_{h=0}^{H-1} \gamma^h \EE_{\tau \sim p(\tau|\theta)}\left[ \left(\log\pi_\theta(a^i_h|s^i_h)\right)^2\right]\\
\leq &4 \lambda^2 \left(\frac{\gamma+1}{(1-\gamma)^3}\right) \sum_{h=0}^{H-1} \gamma^h (\log |\mathcal{A}|)^2\\
\leq & \frac{8 \lambda^2(\log |\mathcal{A}|)^2}{(1-\gamma)^4}.
\end{align*}

Finally, by substituting $\text{Var}(g_1), \text{Var}(g_2)$ and $\text{Var}(g_3)$ into \eqref{eq: variance of g}, it holds that
\begin{align*}
\text{Var}(\hat{\nabla} V^{\theta,H}_\lambda(\rho))\leq
    \frac{12\bar{r}^2+24 \lambda^2(\log |\mathcal{A}|)^2}{(1-\gamma)^4}.
\end{align*}
This completes the proof.
\end{proof}

\section{Other helpful results.}  \label{sec: Theorem thm: asym global convergence}


\begin{lemma} \label{lemma: due to smoothness}
Suppose that $f(x)$ is $\Bar{L}$-smooth, i.e., $ \norm{\nabla f(x) - \nabla f(y)} \leq \Bar{L} \norm{x-y}$.
Given  $0<\eta_t\leq\frac{1}{2\myred{\Bar{L}}}$ for all $t\geq 1$, let $\{x_t\}_{t=1}^T$ {be generated by} $x_{t+1}=x_t+\eta_t u_t$ and let $e_t=u_t-\nabla f(x_t)$. We have
\begin{align*}
f(x_{t+1})\geq& f(x_t)+\frac{\eta_t}{4} \norm{u_t}_2^2-\frac{\eta_t}{2}\norm{e_t}_2^2.
\end{align*}
\end{lemma}
\begin{proof}
Due to the space restriction, please refer to the online version of the paper \cite{dingbeyond} for the complete proofs.
\end{proof}

\begin{IEEEbiography}[{\includegraphics[width=1in,height=1.25in,clip,keepaspectratio]{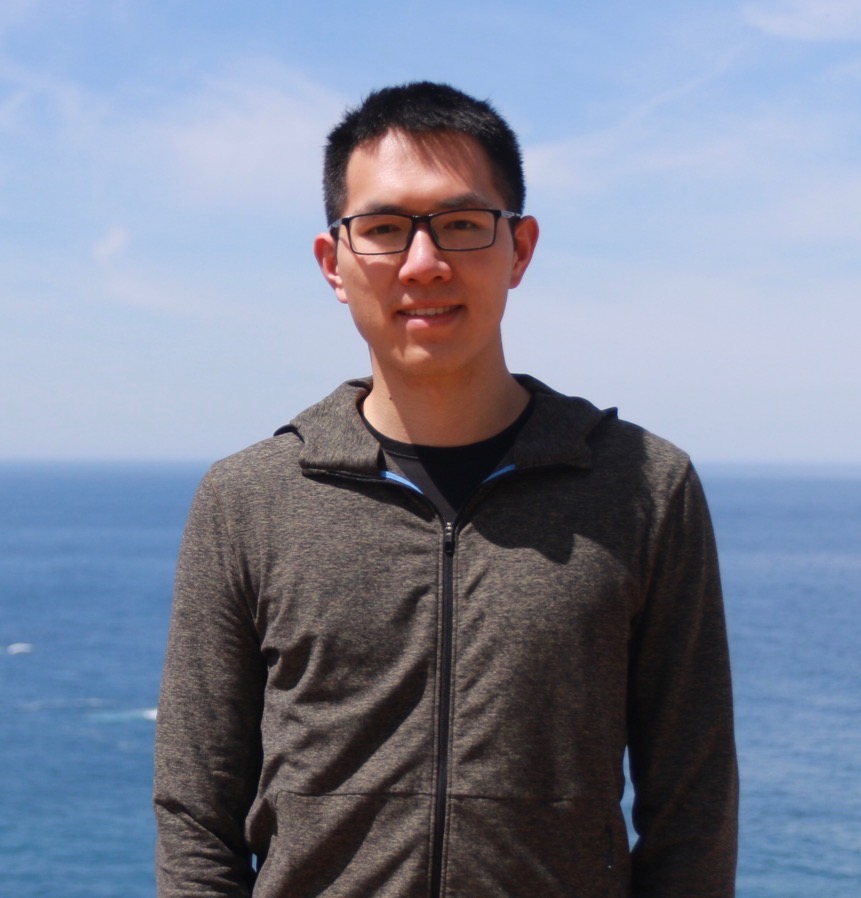}}]{Yuhao Ding} is currently a Ph.D. candidate in Industrial Engineering and Operations Research at the University of California, Berkeley. He has worked on different interdisciplinary problems in optimization and control theory.
He obtained the B.E. degree in Aerospace Engineering from Nanjing University of Aeronautics and Astronautics in 2016, and the M.S. degree in Electrical and Computer Engineering from University of Michigan, Ann Arbor in 2018.
\end{IEEEbiography}

\begin{IEEEbiography}[{\includegraphics[width=1in,height=1.25in,clip,keepaspectratio]{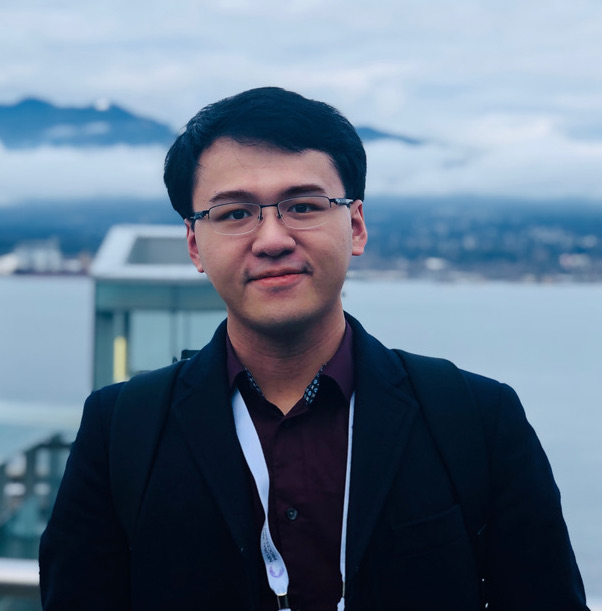}}]{Junzi Zhang} is currently an applied scientist in Amazon.com LLC. He obtained a Ph.D. degree in Computational Mathematics from Stanford University and a B.S. degree in Applied Mathematics from School of Mathematical Sciences, Peking University. His research is focused on the design and analysis of optimization algorithms and software, and extends broadly into the fields of machine learning, causal inference and decision-making systems.
\end{IEEEbiography}

\begin{IEEEbiography}[{\includegraphics[width=1in,height=1.25in,clip,keepaspectratio]{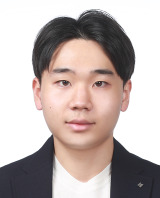}}]{Hyunin Lee} is currently a Ph.D. student in Mechanical Engineering at the University of California, Berkeley. His research is focused on the reinforcement learning and the optimization theory. He obtained the B.S. degree in Mechanical Engineering from Seoul National University in 2022. 
\end{IEEEbiography}

\begin{IEEEbiography}[{\includegraphics[width=1in,height=1.25in,clip,keepaspectratio]{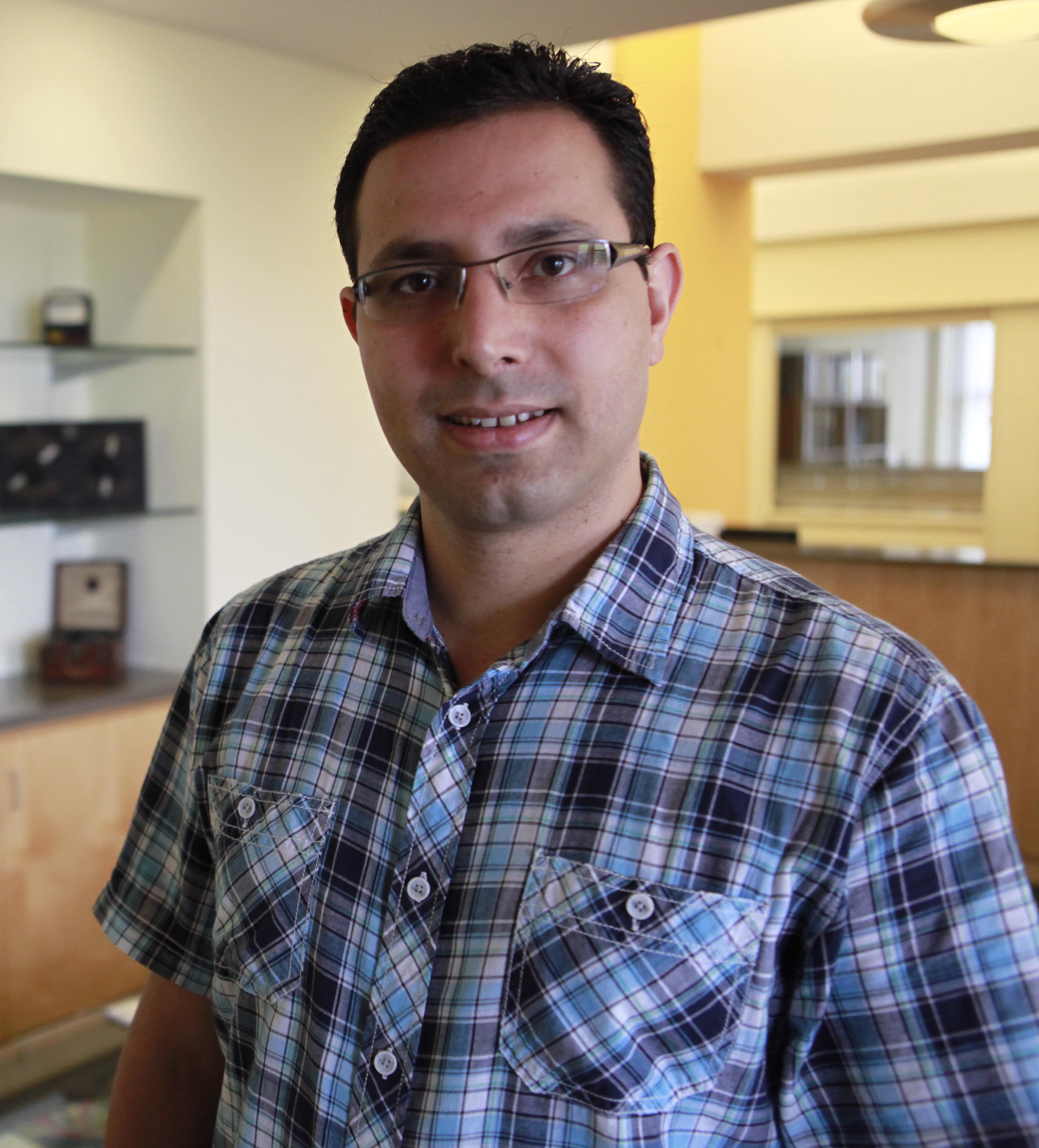}}]{Javad Lavaei} is currently an Associate Professor in the
Department of Industrial Engineering and Operations Research at the University of California, Berkeley.
He has
worked on different interdisciplinary problems in
power systems, optimization theory, control theory,
and data science. 
He is an associate editor of the IEEE Transactions on Automatic Control, the
IEEE Transactions on Smart Grid, and the IEEE Control System
Letters. 
\end{IEEEbiography}





\end{document}